\documentclass[10pt]{article} 
\usepackage[preprint]{tmlr}

\usepackage{amsmath,amsfonts,bm}

\def\eqref#1{equation~\ref{#1}}

\def\1{\bm{1}}

\DeclareMathAlphabet{\mathsfit}{\encodingdefault}{\sfdefault}{m}{sl}
\SetMathAlphabet{\mathsfit}{bold}{\encodingdefault}{\sfdefault}{bx}{n}

\usepackage{hyperref}
\usepackage{url}

\usepackage{amsmath}
\usepackage{amssymb}
\usepackage{amsfonts}
\usepackage{amsthm}
\usepackage{multirow}
\usepackage{booktabs}
\usepackage{adjustbox}  
\usepackage{array}
\usepackage{graphicx}
\usepackage{subcaption}
\usepackage{listings}
\usepackage{xcolor} 

\definecolor{codegreen}{rgb}{0,0.6,0}
\definecolor{codegray}{rgb}{0.5,0.5,0.5}
\definecolor{codepurple}{rgb}{0.58,0,0.82}
\definecolor{backcolour}{rgb}{0.96,0.96,0.96}

\lstdefinestyle{mypython}{
    backgroundcolor=\color{backcolour},  
    commentstyle=\color{codegreen},       
    keywordstyle=\color{magenta}\bfseries,
    numberstyle=\tiny\color{codegray},    
    stringstyle=\color{codepurple},       
    basicstyle=\ttfamily\footnotesize,   
    breakatwhitespace=false,         
    breaklines=true,                      
    captionpos=b,                         
    keepspaces=true,                 
    numbers=left,                         
    numbersep=5pt,                  
    showspaces=false,                
    showstringspaces=false,          
    showtabs=false,                  
    tabsize=4,                            
    language=Python                       
}

\usepackage{pifont}
\usepackage{enumitem}
\usepackage{mathrsfs}

\newtheorem{proposition}{Proposition}
\title{On the Global Photometric Alignment for Low-Level Vision}

\author{\name Mingjia Li\thanks{Equal contribution.}, \name Tianle Du\footnotemark[1], \name Hainuo Wang, \name Qiming Hu, \name Xiaojie Guo\thanks{Corresponding author.} \\
      \addr Tianjin University \\
      \email \{mingjiali, dutianle, hainuo, huqiming\}@tju.edu.cn, xj.max.guo@gmail.com
}

\def\month{MM}  
\def\year{YYYY} 
\def\openreview{\url{https://openreview.net/forum?id=XXXX}} 

\begin{document}

\maketitle

\begin{abstract}
Supervised low-level vision models rely on pixel-wise losses against paired references, yet paired training sets exhibit per-pair photometric inconsistency, say, different image pairs demand different global brightness, color, or white-balance mappings. This inconsistency enters through task-intrinsic photometric transfer (e.g., low-light enhancement) or unintended acquisition shifts (e.g., de-raining), and in either case causes an optimization pathology. Standard reconstruction losses allocate disproportionate gradient budget to conflicting per-pair photometric targets, crowding out content restoration. In this paper, we investigate this issue and prove that, under least-squares decomposition, the photometric and structural components of the prediction-target residual are orthogonal, and that the spatially dense photometric component dominates the gradient energy. Motivated by this analysis, we propose Photometric Alignment Loss (PAL). This flexible supervision objective discounts nuisance photometric discrepancy via closed-form affine color alignment while preserving restoration-relevant supervision, requiring only covariance statistics and tiny matrix inversion with negligible overhead. Across 6 tasks, 16 datasets, and 16 architectures, PAL consistently improves metrics and generalization. The implementation is in the appendix.
\end{abstract}

\section{Introduction}

Pixel-wise supervision underpins most state-of-the-art low-level vision models~\citep{xu2023low, yan2025hvi, shadowhack, mei2024latent, cai2016dehazenet, wang2025modem, 10.1007/978-3-031-72670-5_7}. By regressing network outputs toward paired reference~\footnote{In the paper, we follow the convention of calling reference images ``ground truth,'' though this term is technically inexact for enhancement tasks.} images, models learn complex mappings from degraded inputs to clean targets. Despite its success, this paradigm rests on an implicit assumption: that every pixel-level difference between prediction and target is equally worth fitting. In practice, the prediction-target residual often contains a substantial photometric component, namely global shifts in brightness, color, or white balance, that vary from pair to pair within the training set. We refer to this variation as \emph{per-pair photometric inconsistency}. Because standard reconstruction losses are disproportionately sensitive to global shifts~\citep{11095167}, the gradient signal is dominated by conflicting per-pair photometric targets at the expense of structural restoration.

Per-pair photometric inconsistency enters paired datasets through two distinct sources. The first is \emph{task-intrinsic}: in low-light enhancement~\citep{DBLP:journals/tip/GuoLL17, zhang2019kindling} and underwater enhancement~\citep{8917818}, the ground truth intentionally differs from the input in brightness and color, yet different pairs demand different photometric mappings depending on capture conditions and photographer intent. Standard pixel-wise losses allocate most of their gradient to this large but pair-specific photometric signal, leaving content restoration, say, contaminated texture and structure, underrepresented in the gradient.

\begin{figure}
  \includegraphics[width=\textwidth]{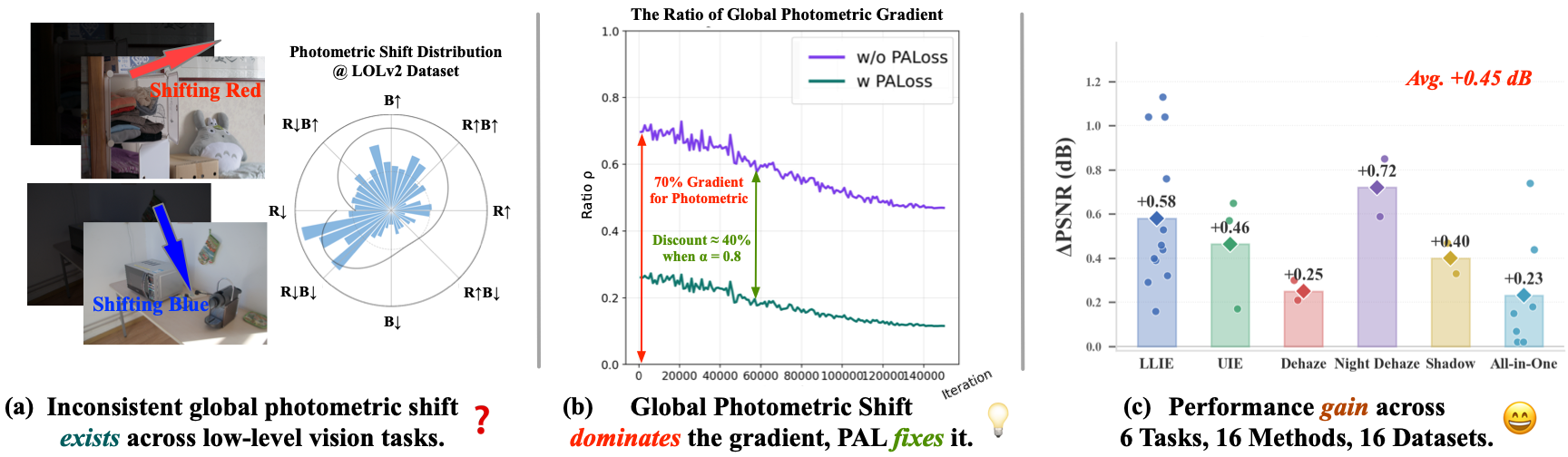}
  \caption{An Illustration of our work. (a) We identified that there exist inconsistent global photometric shift across the paired training datasets. (b) As the photometric shift dominates the gradient, learning texture and structure content from the training data becomes difficult. Our PAL helps rebalance the gradient. (c) Equipped with our PAL, the performance gained across 6 tasks, 16 methods, and 16 datasets, with an average gain in PSNR of 0.45dB across 6 tasks.}
  \label{fig:teaser}
\end{figure}

The second source is \emph{acquisition-induced}: in image dehazing~\citep{cai2016dehazenet} and deraining~\citep{10.1007/978-3-031-72670-5_7}, the restoration target should not differ photometrically from the desired content, yet paired-data acquisition introduces exposure, white-balance, or color-temperature variations that differ from pair to pair. Because different pairs deviate in random directions, the model receives contradictory supervision about whether and how to alter scene color, wasting capacity on what amounts to photometric label noise.

Although the two sources differ in origin and scale, they produce the same optimization pathology, \emph{i.e.}, the network exhausts its gradient budget resolving per-pair photometric conflicts rather than learning content restoration. The severity depends on the magnitude of the inconsistency and the dataset size, but the underlying mechanism is identical. Shadow removal~\citep{mei2024latent} illustrates this clearly: inside shadow regions, photometric transfer is intrinsic to the task, while outside shadow regions, residual acquisition deviation provides contradictory supervision. Both coexist within a single image as instances of the same per-pair inconsistency, motivating a general formulation rather than task-specific fixes.

Existing strategies mitigate the issue partially. Perceptual and adversarial losses~\citep{johnson2016perceptual, goodfellow2014generative} can be seen as implicit photometric robustness~\citep{zhang2018unreasonable} through deep feature matching, but at substantial computational cost and with only indirect supervision. Alternative color spaces~\citep{lore2017llnet, yan2025hvi, shadowhack, DBLP:journals/ijcv/GuoH23} are successful in many domains, improving the performance, but they are rather reorganizing the problem without eliminating it. Moreover, they are task-specific and typically require sophisticated co-design with modeling. A generalized mechanism that explicitly discounts photometric discrepancy from the supervision signal remains missing.

In this paper, we present Photometric Alignment Loss (PAL), a task‑agnostic approach for globally uniform photometric inconsistencies that addresses per-pair photometric inconsistency. An overview of our work can be found in Figure~\ref{fig:teaser}. PAL models the photometric discrepancy between prediction and target as a global affine color transformation and solves for it in closed form. The reconstruction loss is then computed on the aligned residual, so that restoration-relevant content can better drive the optimization. PAL can also be extended for spatially varying shifts with a mask. We validate PAL across 6 tasks, 16 datasets, and 16 architectures, demonstrating consistent improvements.  In summary, our contributions are:
\begin{itemize}
    \item We identify per-pair photometric inconsistency as a unified source of optimization distortion in paired low-level vision, and show that it arises from both task-intrinsic and acquisition-induced origins.
    \item We propose PAL, a closed-form color alignment loss that discounts nuisance photometric discrepancy from the gradient while preserving content supervision, with negligible computational overhead.
    \item We provide extensive real-task validation across 6 low-level vision tasks, 16 datasets, and 16 architectures, demonstrating improvements in fidelity metrics and generalization.
\end{itemize}

\section{Related Work}

\subsection{Supervised Learning for Low-level Vision}

Paired supervision has become the dominant training paradigm across a broad range of low-level vision tasks, yet each task family exhibits its own form of vulnerability to photometric inconsistency.

In image restoration tasks such as dehazing~\citep{cai2016dehazenet, qin2020ffa, song2023vision} and deraining~\citep{li2019single, zamir2020learning, zamir2022restormer}, the objective is to recover clean content without altering the scene photometry. However, paired training data for these tasks are typically generated from synthetic degradation pipelines or collected under controlled but imperfectly matched conditions. Subtle differences in camera exposure, white balance, or tone mapping between the degraded input and its clean counterpart introduce photometric shifts that are artifacts of the acquisition process rather than part of the degradation to be removed. Models trained with strict pixel-wise losses inherit these shifts as spurious supervision targets. This problem is further compounded in all-in-one restoration frameworks~\citep{wang2025modem, 10.1007/978-3-031-72670-5_7}, which train a single model across multiple degradation types such as rain, snow, and haze simultaneously. Because each constituent dataset is collected under different imaging conditions with its own photometric profile, mixing them amplifies the inconsistency. The model receives contradictory photometric supervision not only across image pairs, but also across tasks, making the consensus even more challenging.

Image enhancement tasks, including low-light enhancement~\citep{wei2018deep, zhang2019kindling, xu2023low, cai2023retinexformer, yan2025hvi} and underwater image enhancement~\citep{8917818, liu2022boths, islam2020fast}, face the converse challenge. Here, the ground truth intentionally differs from the input in brightness and color, making photometric transfer an integral part of the objective. A substantial body of work has developed architectures that range from Retinex-inspired decomposition networks~\citep{wei2018deep, zhang2019kindling, fu2020learning} to encoder-decoder~\citep{xu2023low, zamir2020learning} and transformer-based models~\citep{cai2023retinexformer, wang2022uformer, zamir2022restormer}. Despite their architectural diversity, these methods universally rely on pixel-wise reconstruction losses and therefore the easy global photometric gap dominates the gradient for both scenarios, suppressing the signal for content recovery.

There also exist tasks where intentional and unintentional photometric discrepancies coexist in a single training pair, an example of which is shadow removal~\citep{shadowhack, GuoHLCW23, GuoWYHWPW23, mei2024latent}. To be specific, the shadow regions require photometric correction, while the non-shadow regions should ideally remain unchanged. Paired datasets for this task are constructed by photographing scenes with and without cast shadows. 
During the capture of these datasets, the non-shadow area will inevitably introduce global photometric variation as well. This spatial coexistence makes shadow removal a natural testbed for methods that aim to handle both sources of inconsistency.

\begin{figure*}[t]
    \centering

    \includegraphics[width=0.16\textwidth]{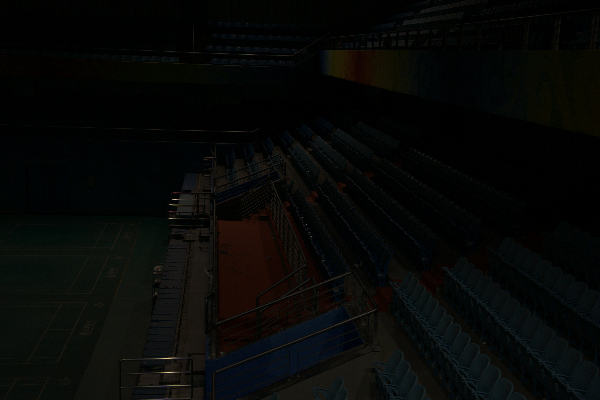}
    \includegraphics[width=0.16\textwidth]{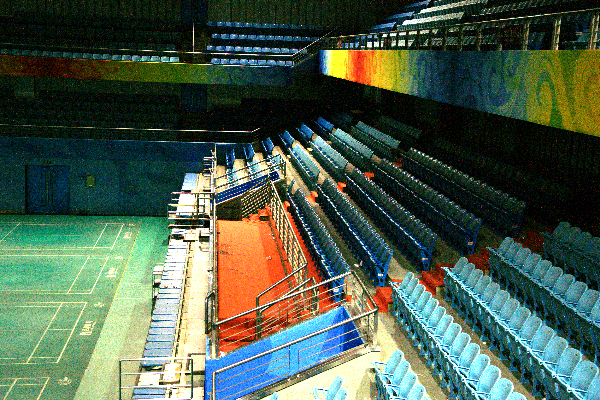}
    \includegraphics[width=0.16\textwidth]{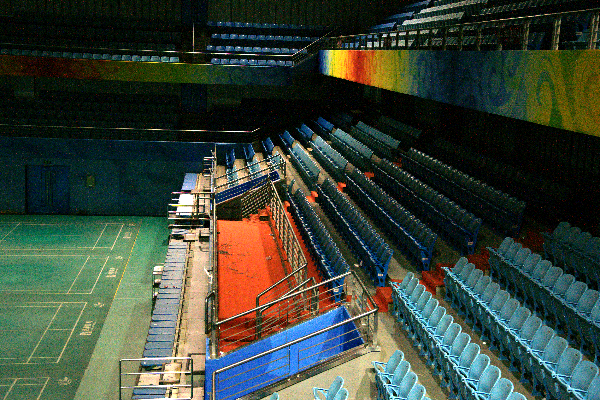}
    \includegraphics[width=0.16\textwidth]{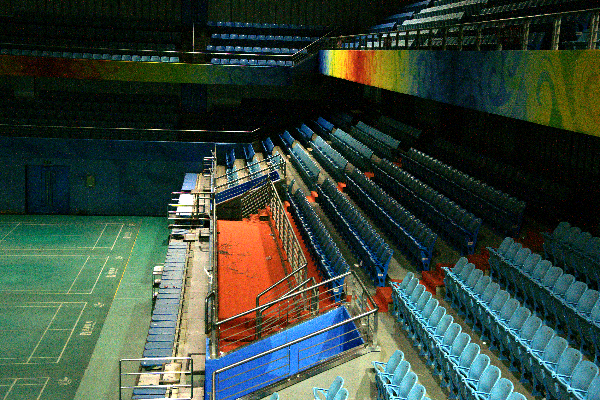}
    \includegraphics[width=0.16\textwidth]{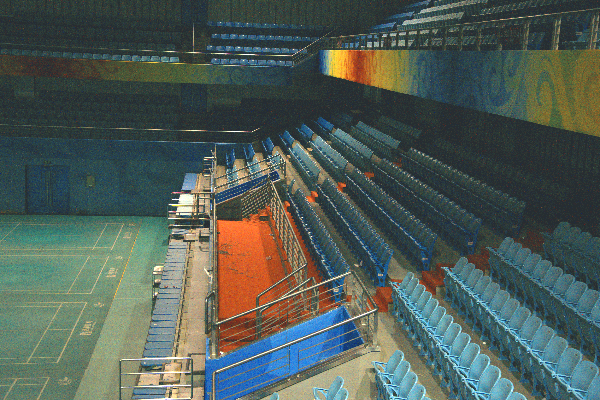}
    \includegraphics[width=0.16\textwidth]{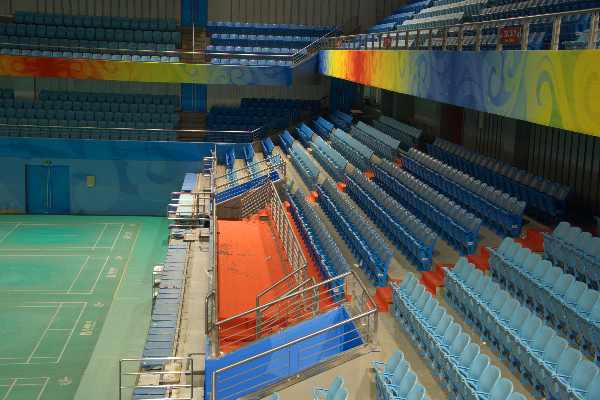}

    \includegraphics[width=0.16\textwidth]{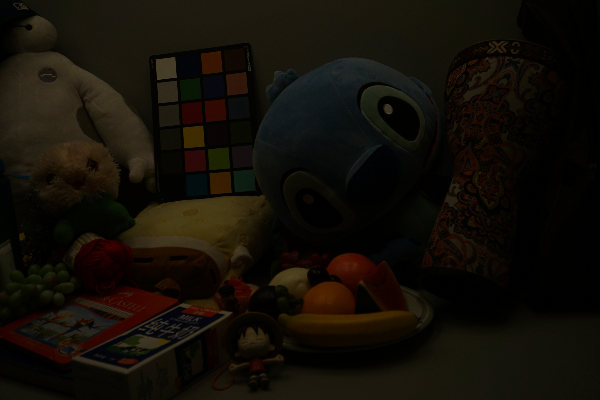}
    \includegraphics[width=0.16\textwidth]{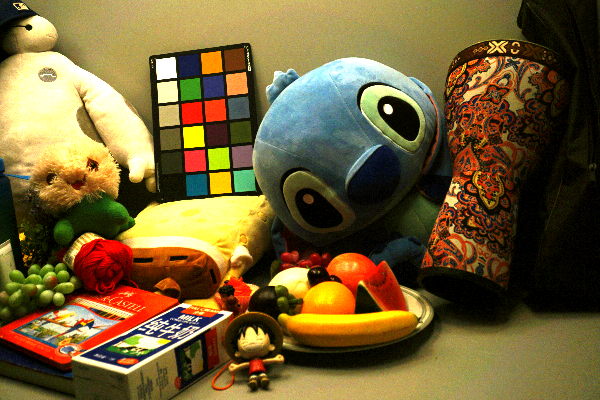}
    \includegraphics[width=0.16\textwidth]{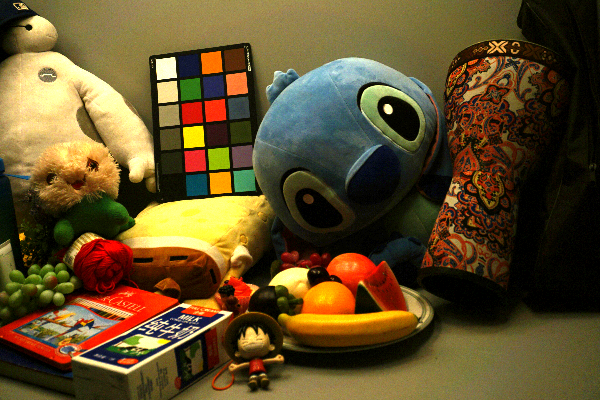}
    \includegraphics[width=0.16\textwidth]{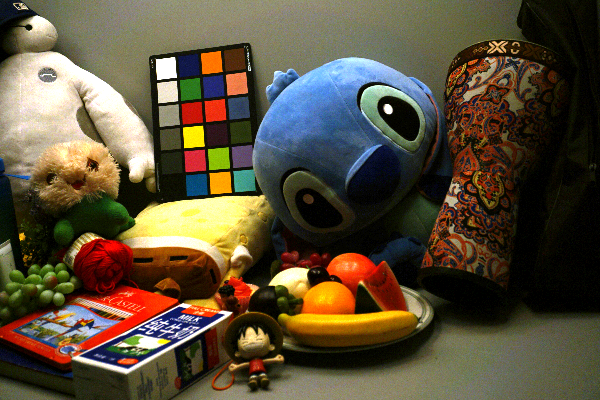}
    \includegraphics[width=0.16\textwidth]{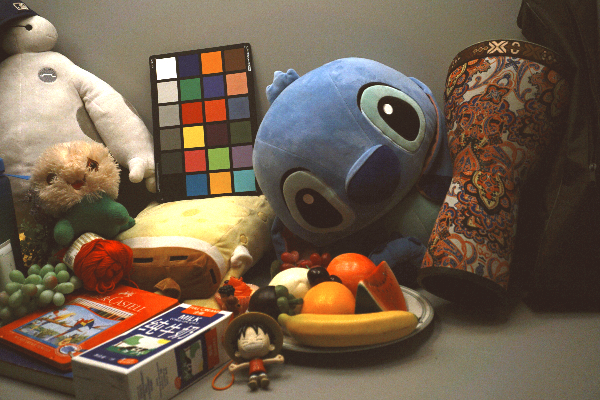}
    \includegraphics[width=0.16\textwidth]{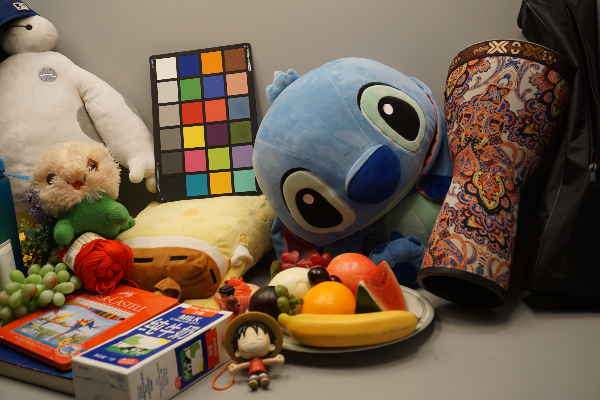}

        \includegraphics[width=0.16\textwidth]{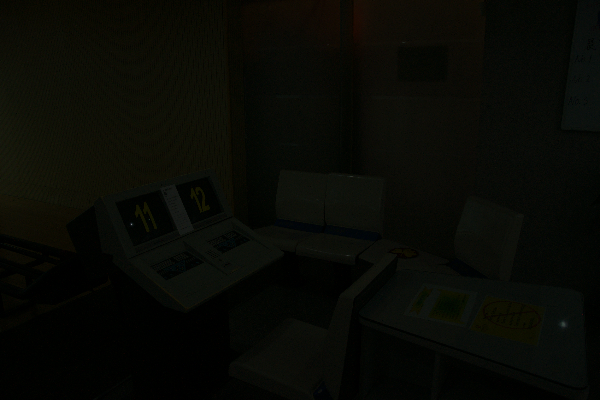}
    \includegraphics[width=0.16\textwidth]{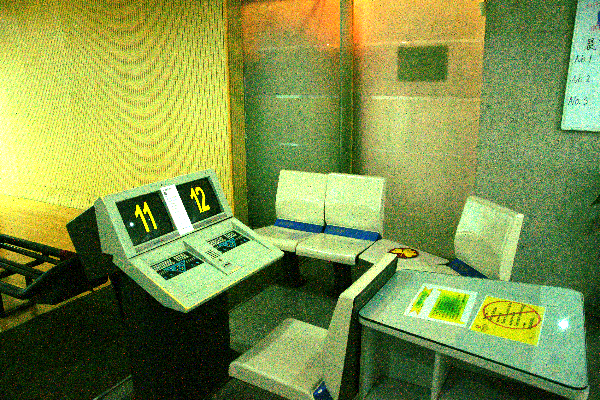}
    \includegraphics[width=0.16\textwidth]{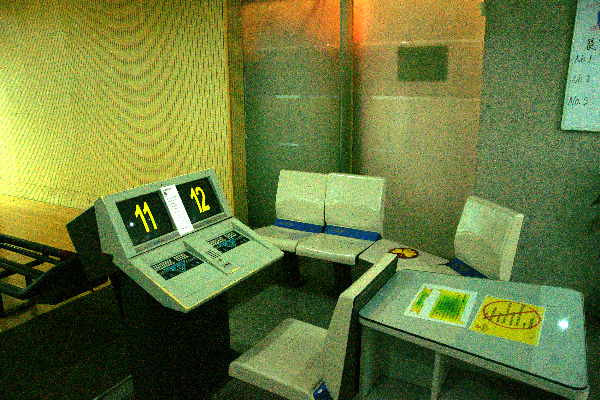}
    \includegraphics[width=0.16\textwidth]{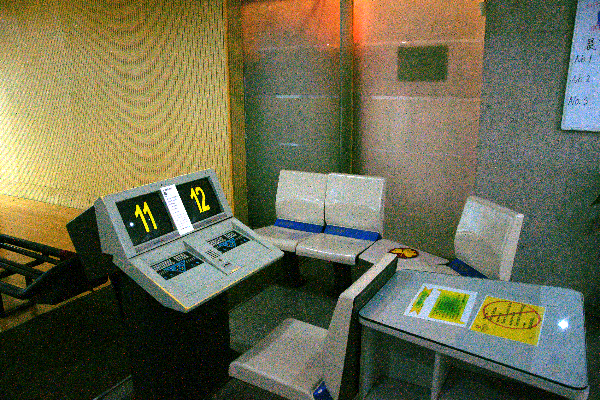}
    \includegraphics[width=0.16\textwidth]{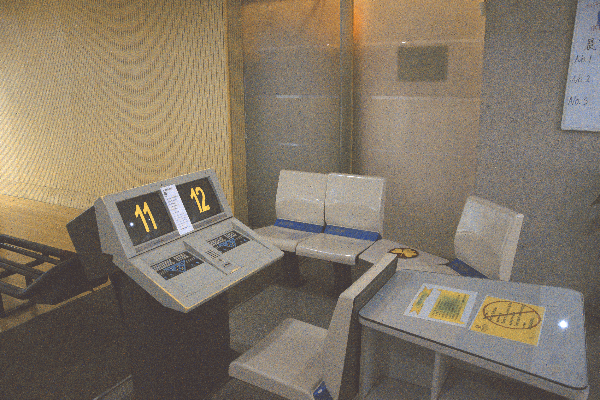}
    \includegraphics[width=0.16\textwidth]{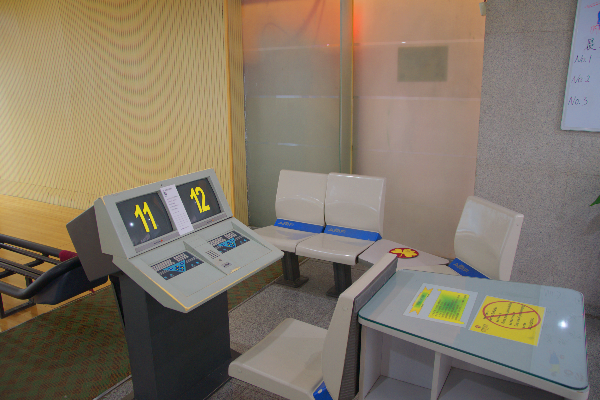}
   \begin{subfigure}{0.16\textwidth}
        \centering
        \subcaption{Input}
    \end{subfigure}
    \begin{subfigure}{0.16\textwidth}
        \centering
        \subcaption{Chan. Mean}
    \end{subfigure}
    \begin{subfigure}{0.16\textwidth}
        \centering
        \subcaption{Opti. Scalar}
    \end{subfigure}
    \begin{subfigure}{0.16\textwidth}
        \centering
        \subcaption{Opti. Diagonal}
    \end{subfigure}
    \begin{subfigure}{0.16\textwidth}
        \centering
        \subcaption{PAL}
    \end{subfigure}
    \begin{subfigure}{0.16\textwidth}
        \centering
        \subcaption{GT}
    \end{subfigure}

    \caption{Comparison of alignment families applied to Low-light image pairs. The optimal transforms are computed in closed form and applied. Channel-wise mean~(b) and Optimal Scalar~(c) cannot correct color-temperature shifts . Optimal Diagonal~(d) handles per-channel gain but not cross-channel coupling. Only PAL's full affine model~(e) closely matches the reference~(f).}
    \label{fig:photometric_analysis_grid}

\end{figure*}

\subsection{Loss Functions for Pixel-wise Supervision}

Perceptual losses~\citep{johnson2016perceptual,zhang2018unreasonable} shift supervision from pixel space to deep feature space by computing distances between VGG activations of enhanced and reference images. Because these features are learned to be invariant to low-level photometric variations, the loss becomes more robust to exact brightness and color values, focusing instead on semantic and structural content. Similarly, adversarial losses~\citep{goodfellow2014generative,isola2017image} train discriminators to distinguish real from enhanced images, encouraging outputs that lie on the manifold of natural images regardless of specific photometric properties. While these approaches significantly improve perceptual realism and provide implicit photometric robustness, they introduce substantial computational overhead, require careful hyperparameter tuning, and can produce characteristic artifacts~\citep{ledig2017photo,blau2018perception}. Moreover, they provide only indirect supervision, and the network must implicitly learn to ignore photometric variations rather than having them explicitly removed from the supervision signal.

A related family of techniques, style-transfer losses such as Gram-matrix matching~\citep{gatys2016image} and AdaIN statistics alignment~\citep{huang2017arbitrary}, also leverage global feature statistics. However, these methods differ from PAL in both purpose and mechanism. Style losses operate in deep feature space (e.g., VGG activations) and serve as \emph{additional supervision objectives} that encourage the network output to match a reference style. They add a constraint to the optimization. PAL, by contrast, operates directly in the RGB pixel space and serves as a \emph{loss modification}: rather than imposing a new target, it removes per-pair photometric nuisance from the existing pixel-wise supervision signal via closed-form affine regression, so that the residual gradient is redirected toward structural content. In short, style losses push outputs toward a desired distribution, whereas PAL subtracts a nuisance component from the training objective.

Another strategy decouples intensity from chrominance by operating in color spaces like HSV, YUV, or Lab~\citep{lore2017llnet, shadowhack, DBLP:journals/ijcv/GuoH23}. Recent work has proposed learnable or customized color spaces like HVI~\citep{yan2025hvi} or rectifed latent space~\citep{Li2025GenSIRR, liu2025latent}, specifically designed for a better operation space. However, color/latent space conversions introduce their own challenges. These methods can be non-linear and often require specialized architectures that limit their applicability. Furthermore, they do not solve the photometric inconsistency problem, but reorganize it into different, non-optimal channels with task-specific or even model-specific design. This undermines the generalizability to unknown datasets and tasks. In the low-light enhancement community, GT-Mean~\citep{zhang2019kindling, liao2025gtmean} is also proposed to align the global lightness. However, it is biased and can not capture the full global photometric clue. As a result, it is limited to the domain of low-light image enhancement. Drawing on classical color science~\citep{DBLP:journals/pami/FinlaysonHH01,DBLP:journals/tip/BarnardCF02}, we recognize that photometric relationships between images involve coupled color channels. White balance creates off-diagonal terms, while exposure affects channels non-uniformly. We analyze and derive a least-squares estimator for the linear color transformation, providing a more accurate and theoretically principled alignment framework that benefits training and improves generalization capability.

\section{Problem Analysis and Method}

In this section, we provide theoretical analysis and empirical evidence, then derive PAL as a remedy.

\begin{figure}
    \centering
    \includegraphics[width=0.7\linewidth]{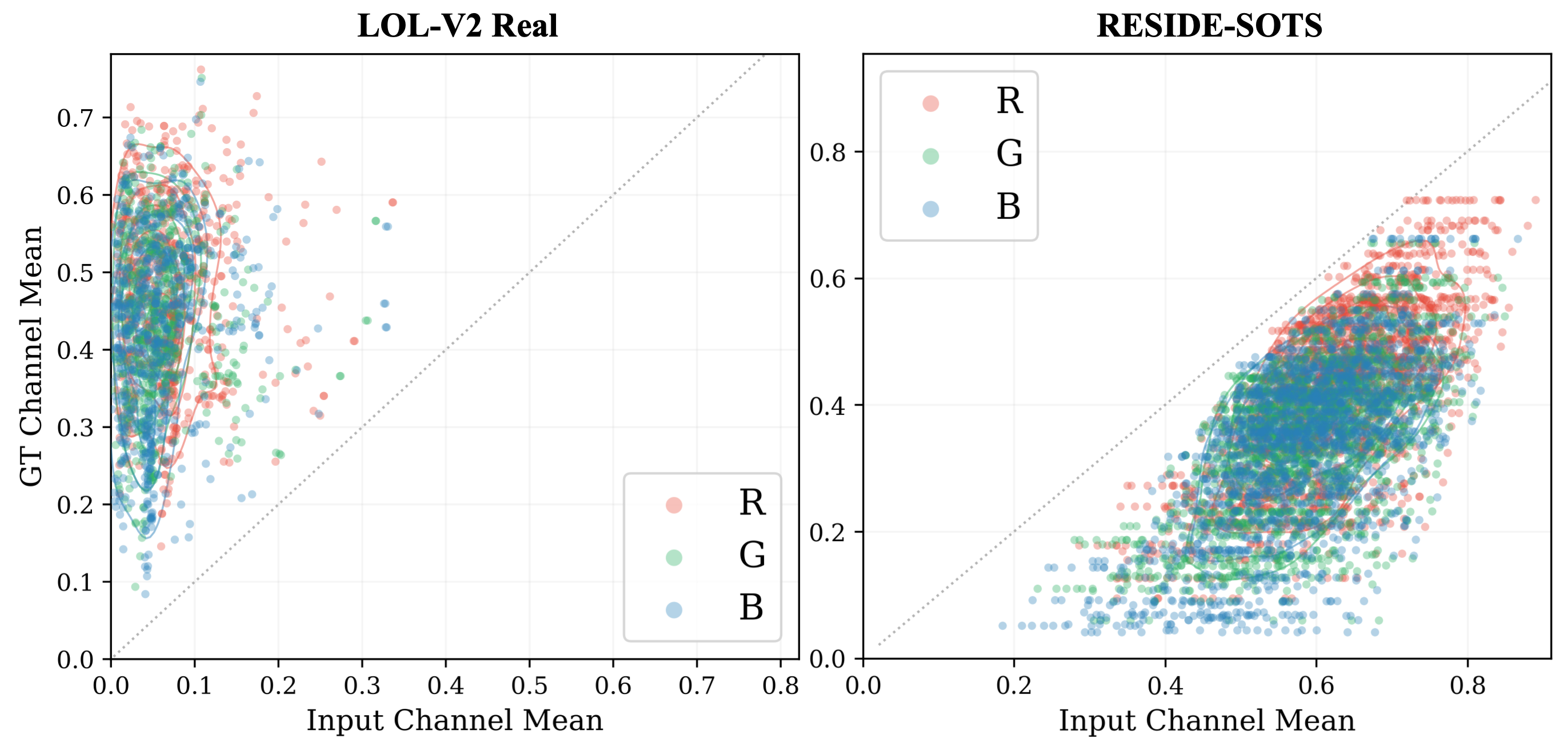}
    \caption{Per-pair photometric scatter plots. Each point represents one training pair, with per-channel (R/G/B) input mean on the $x$-axis, and GT mean on the $y$-axis. In both LOL-v2 and RESIDE-SOTS cases, the wide per-pair spread means a pixel-wise loss receives conflicting photometric supervision.}
    \label{fig:perpair_scatter}
   \vspace{-10pt}
\end{figure}

\subsection{Evidence of Photometric Inconsistency}

Paired low-level vision supervision regresses a prediction $\mathbf{\hat{I}}$ toward a target $\mathbf{I}_{\text{gt}}$ with a pixel-wise loss. This supervision is well-posed only when every prediction-target residual reflects restoration-relevant content alone. In practice, the residual also contains a photometric component, \emph{i.e.}, global shifts in brightness, color, or white balance, that vary from pair to pair within the training set. Because no single photometric mapping satisfies all pairs simultaneously, the pixel-wise loss receives conflicting supervision, as it tries to fit pair-specific photometric targets that are mutually contradictory, leaving an conflict in the gradient signal.

To expose this variation, we compute the per-channel mean brightness of each input and its ground truth across two representative datasets and plot them against each other in Figure~\ref{fig:perpair_scatter}. If photometric consistency held, all points would collapse onto a single line ($y{=}kx{+}b$). Instead, both panels show broad scatter. In LOLv2-Real (left), points sit far from the $y{=}x$ diagonal and spread widely: different pairs demand different brightness gains, and different color channels deviate by different amounts, indicating pair-specific color-temperature and white-balance shifts rather than a uniform brightness scale. In RESIDE-SOTS (right), the task should not alter scene photometry, yet the point cloud still scatters around the diagonal rather than concentrating on a single trajectory. Regardless of whether the photometric gap is large (enhancement) or small (restoration), the per-pair inconsistency is present and injects conflicting targets into the loss. This conflict has a direct consequence on optimization. It shifts every pixel uniformly since the photometric discrepancy is spatially \textit{dense}, while structural differences (textures, edges) are spatially \textit{sparse}. As a result, the photometric component dominates the gradient energy budget. 
\begin{figure}
    \centering
    \includegraphics[width=0.7\linewidth]{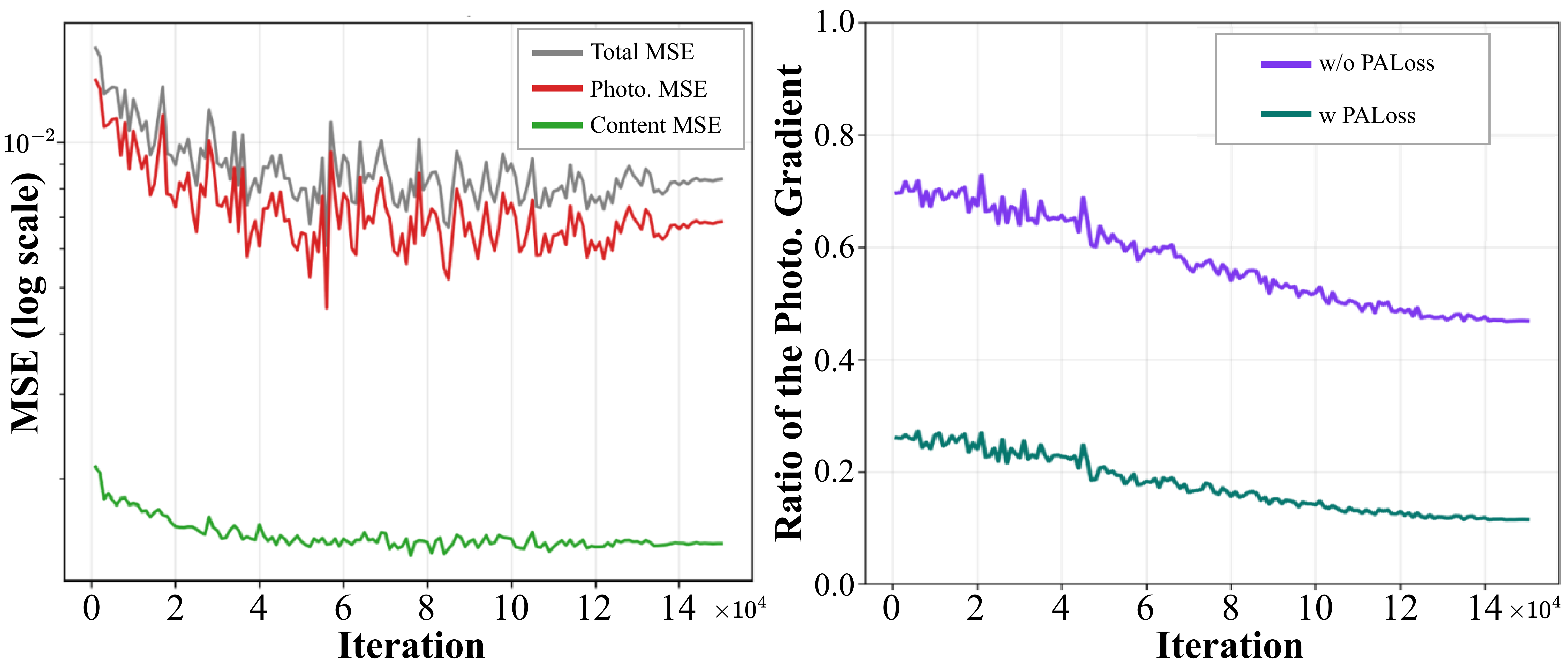}
    \caption{(Left) Decomposed photometric/content error on validation set. (Right) Gradient ratio $\rho$. Plots are sampled from a Retinexformer trained on LOL-V1 every 1000 steps.}
    \label{fig:gradient_stat}
    \vspace{-10pt}
\end{figure}

\subsection{Gradient Dominance of Photometric Error}

As Figure~\ref{fig:perpair_scatter} shows, the prediction-target residual contains a per-pair photometric component. We now prove that, under pixel-wise MSE losses, this component dominates the gradient budget, crowding out the structural supervision that actually drives restoration quality.

\noindent\textbf{Residual decomposition.}
Let $\hat{\mathbf{I}}^{(i)}, \mathbf{I}_{\text{gt}}^{(i)} \in \mathbb{R}^3$ denote the prediction and target at pixel $i$, and let $({\mathbf{C}}^{*}, {\mathbf{b}}^{*})$ be the least-squares affine alignment that minimizes $\sum_{i=1}^{N}\|\mathbf{C}\hat{\mathbf{I}}^{(i)}+\mathbf{b}-\mathbf{I}_{\text{gt}}^{(i)}\|^{2}$.
The per-pixel residual decomposes as
\begin{equation}
    \mathbf{I}_{\text{gt}}^{(i)} - \hat{\mathbf{I}}^{(i)}
    = \underbrace{(\mathbf{C}^{*}-\mathbf{E})\,\hat{\mathbf{I}}^{(i)}+\mathbf{b}^{*}}_{\boldsymbol{\Delta}_{p}^{(i)}\;\text{(photometric)}}
    + \underbrace{\mathbf{I}_{\text{gt}}^{(i)}-\mathbf{C}^{*}\hat{\mathbf{I}}^{(i)}-\mathbf{b}^{*}}_{\boldsymbol{\Delta}_{s}^{(i)}\;\text{(structural)}}.
    \label{eq:residual_decomp}
\end{equation}

\begin{proposition}[Loss decomposition]
\label{prop:pythag}
The pixel-wise MSE decomposes exactly into a photometric term and a structural term with zero cross-term:
\begin{equation}
    \sum_{i}\bigl\|\mathbf{I}_{\emph{gt}}^{(i)}-\hat{\mathbf{I}}^{(i)}\bigr\|^{2}
    = \sum_{i}\bigl\|\boldsymbol{\Delta}_{p}^{(i)}\bigr\|^{2}
    + \sum_{i}\bigl\|\boldsymbol{\Delta}_{s}^{(i)}\bigr\|^{2}.
    \label{eq:pythag}
\end{equation}
\end{proposition}

\begin{proof}

\noindent The pixel-wise MSE can be described as
\begin{equation}
    \sum_{i}\bigl\|\mathbf{I}_{\emph{gt}}^{(i)}-\hat{\mathbf{I}}^{(i)}\bigr\|^{2}
    = \sum_{i}\bigl\|\boldsymbol{\Delta}_{p}^{(i)}\bigr\|^{2}
    + \sum_{i}\bigl\|\boldsymbol{\Delta}_{s}^{(i)}\bigr\|^{2} +   \textstyle\sum_{i}\langle\boldsymbol{\Delta}_{p}^{(i)},\boldsymbol{\Delta}_{s}^{(i)}\rangle.
    \label{eq:original_form}
\end{equation}
Noticing that the first-order optimality conditions of the least-squares affine fit yield
\begin{equation}
    \textstyle\sum_{i}\boldsymbol{\Delta}_{s}^{(i)}=\mathbf{0},\qquad
    \textstyle\sum_{i}\boldsymbol{\Delta}_{s}^{(i)}\,\hat{\mathbf{I}}^{(i)\!\top}=\mathbf{0}.
    \label{eq:opt_cond}
\end{equation}
Expanding the cross-term:
\begin{align}
    \textstyle\sum_{i}\langle\boldsymbol{\Delta}_{p}^{(i)},\boldsymbol{\Delta}_{s}^{(i)}\rangle
    &= \mathrm{tr}\!\Bigl[(\mathbf{C}^{*}\!-\!\mathbf{E})^{\!\top}\!\textstyle\sum_{i}\boldsymbol{\Delta}_{s}^{(i)}\hat{\mathbf{I}}^{(i)\!\top}\Bigr]
       + \mathbf{b}^{*\!\top}\!\textstyle\sum_{i}\boldsymbol{\Delta}_{s}^{(i)} \nonumber\\
    &= \mathrm{tr}\!\bigl[(\mathbf{C}^{*}\!-\!\mathbf{E})^{\!\top}\!\cdot\mathbf{0}\bigr]
       + \mathbf{b}^{*\!\top}\!\cdot\mathbf{0} = 0.
    \label{eq:cross_zero}
\end{align}

\end{proof}

\noindent\textbf{Implication for gradient budget.} For a standard Mean Squared Error (MSE) loss $\mathcal{L}_{\text{MSE}} = \frac{1}{N} \sum_{i}\|\mathbf{I}_{\text{gt}}^{(i)}-\hat{\mathbf{I}}^{(i)}\|^{2}$, the per-pixel gradient with respect to the prediction is $-\frac{2}{N}(\boldsymbol{\Delta}_{p}^{(i)}+\boldsymbol{\Delta}_{s}^{(i)})$. Leveraging the exact orthogonality established above, the total gradient energy perfectly splits into two independent budgets:
\begin{equation}
    \sum_{i}\bigl\|\nabla_{\hat{\mathbf{I}}^{(i)}}\mathcal{L}_{\text{MSE}}\bigr\|^{2}
    = \underbrace{\frac{4}{N^2}\sum_{i}\|\boldsymbol{\Delta}_{p}^{(i)}\|^{2}}_{\mathcal{E}_{\text{phot}}}
    \;+\;
    \underbrace{\frac{4}{N^2}\sum_{i}\|\boldsymbol{\Delta}_{s}^{(i)}\|^{2}}_{\mathcal{E}_{\text{struct}}}.
    \label{eq:grad_energy}
\end{equation}
Let $\rho = \mathcal{E}_{\text{phot}}/(\mathcal{E}_{\text{phot}}+\mathcal{E}_{\text{struct}})$ denote the photometric fraction of the total gradient energy. The critical issue lies in the spatial density of these errors. When a macroscopic photometric mismatch occurs (e.g., a global brightness shift), the photometric error is \textit{dense}, accumulating across all $N$ pixels. Its overall gradient energy $\mathcal{E}_{\text{phot}}$ therefore scales proportionally to $1/N$. In contrast, the structural error $\boldsymbol{\Delta}_{s}^{(i)}$ is \textit{sparse}, confined to a small subset of $M$ localized pixels around misaligned textures or edges ($M \ll N$). Its gradient energy $\mathcal{E}_{\text{struct}}$ only accumulates over these $M$ pixels, scaling as $M/N^2$. 

\begin{figure}[t]
    \centering
    \setlength{\tabcolsep}{1pt}
    \renewcommand{\arraystretch}{0.5}
    \begin{tabular}{cccc}
        \includegraphics[width=0.245\linewidth]{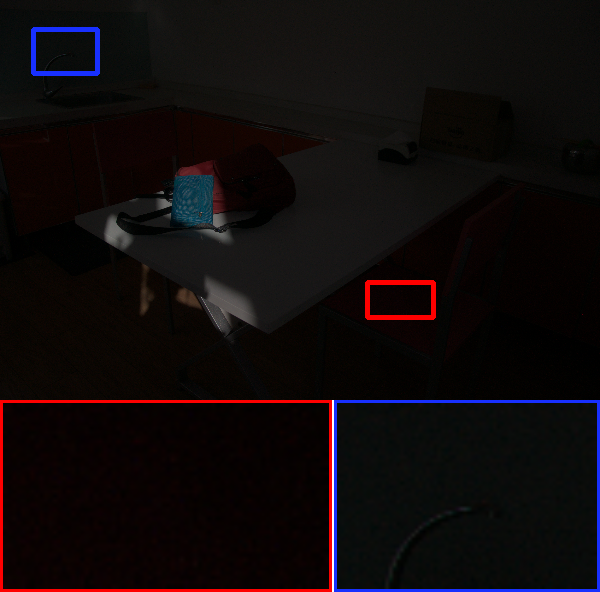} &
        \includegraphics[width=0.245\linewidth]{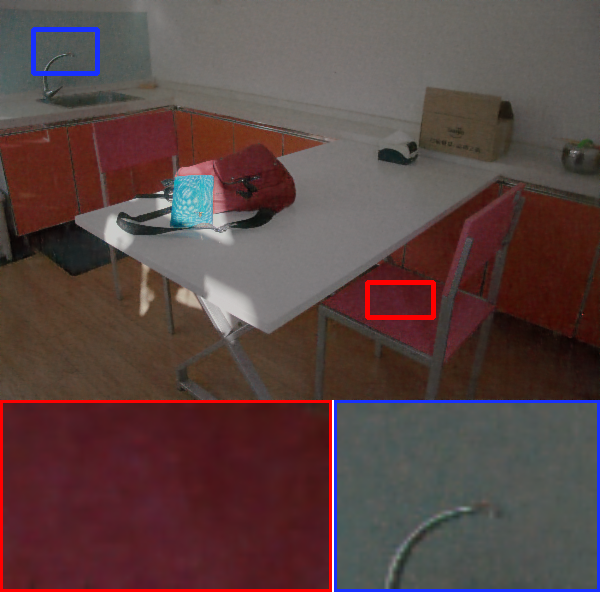} &
        \includegraphics[width=0.245\linewidth]{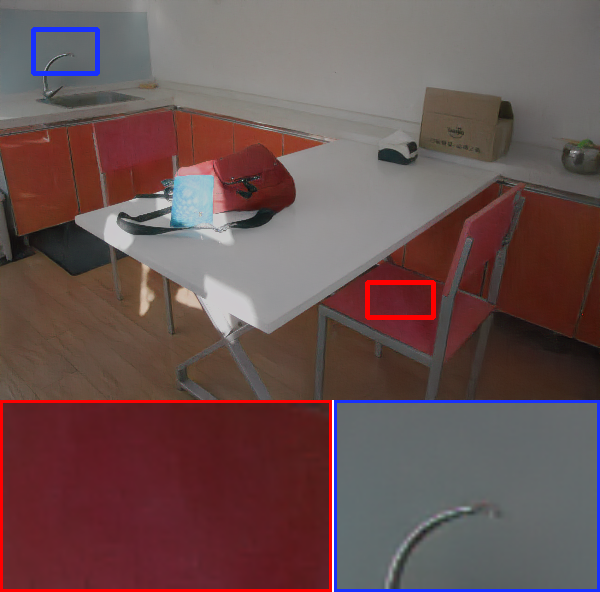} &
        \includegraphics[width=0.245\linewidth]{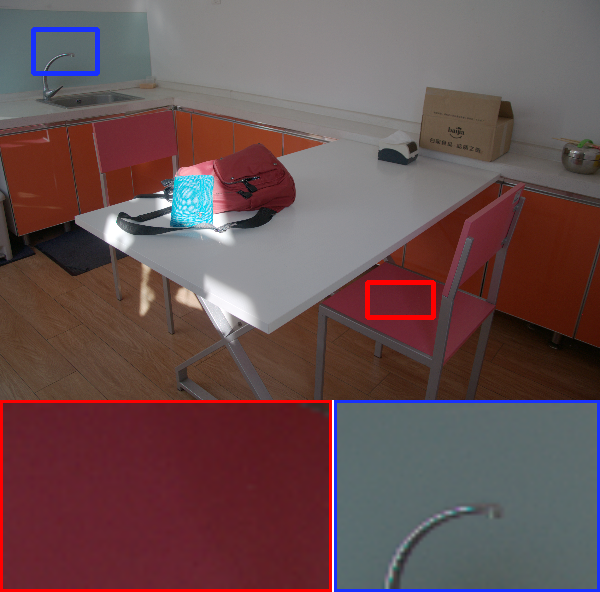} \\
        \includegraphics[width=0.245\linewidth]{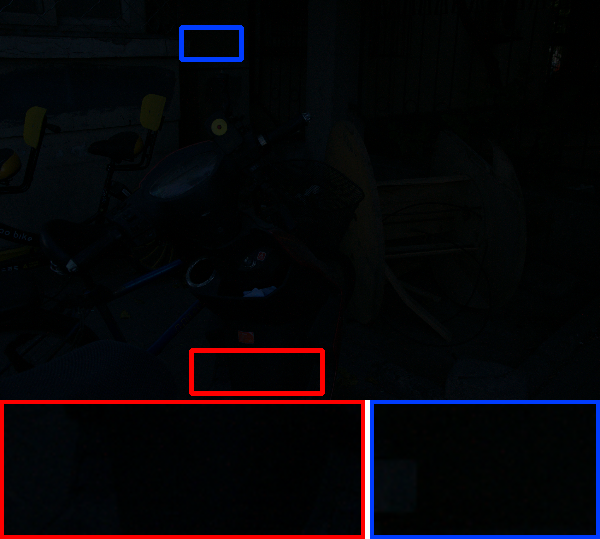} &
        \includegraphics[width=0.245\linewidth]{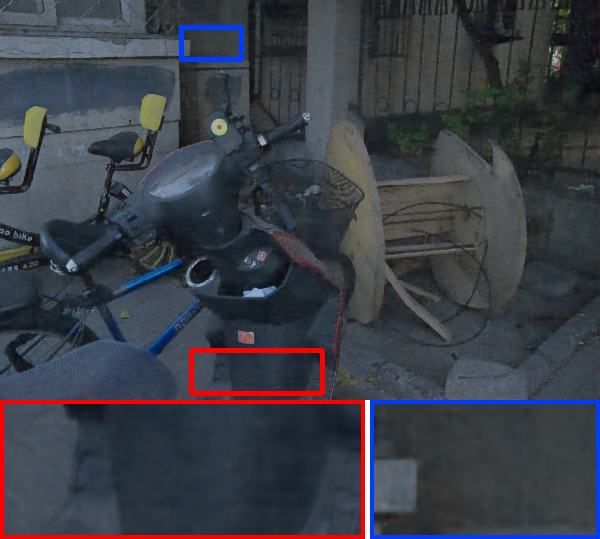} &
        \includegraphics[width=0.245\linewidth]{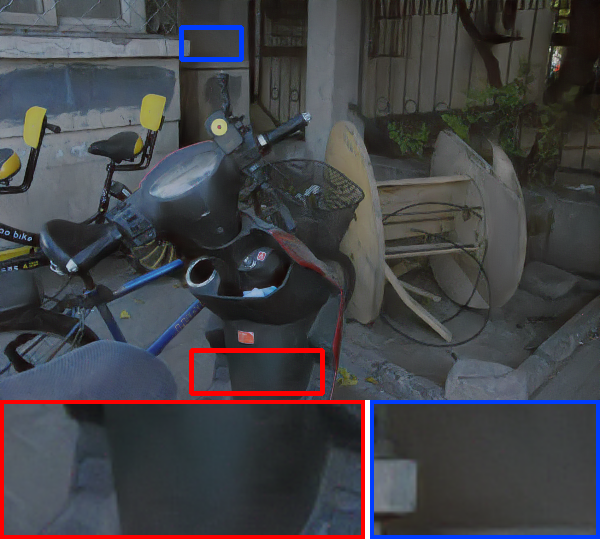} &
        \includegraphics[width=0.245\linewidth]{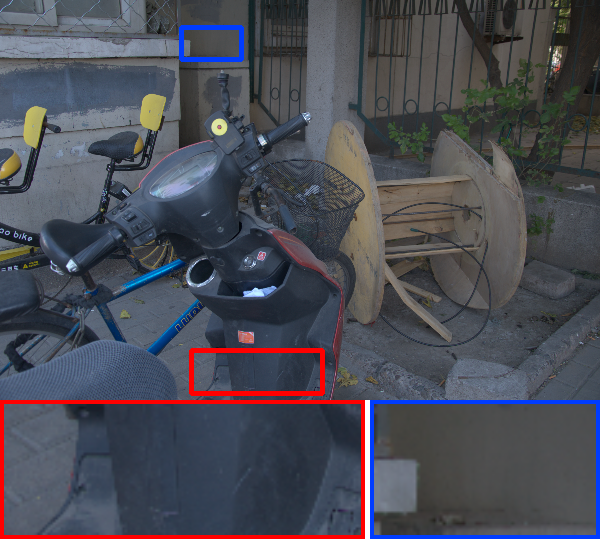} \\
        \addlinespace[2pt]
        Input & CIDNet & CIDNet+PAL & GT
    \end{tabular}
    \caption{Qualitative comparisons on LLIE (CIDNet on LOLv2-real). PAL produces more natural colors.}
    \label{fig:llie_nightdehaze}
\end{figure}

\begin{figure}[t]
    \centering
    \setlength{\tabcolsep}{1pt}
    \renewcommand{\arraystretch}{0.5}
    \begin{tabular}{cccc}
        \includegraphics[width=0.245\linewidth]{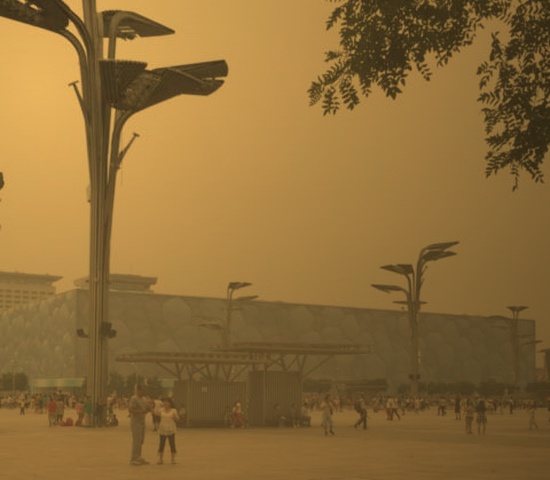} &
        \includegraphics[width=0.245\linewidth]{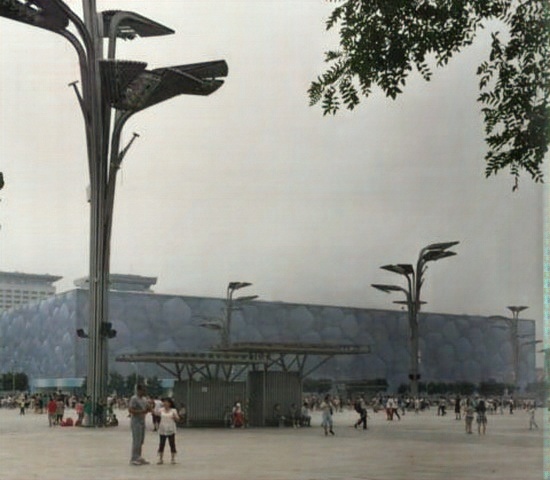} &
        \includegraphics[width=0.245\linewidth]{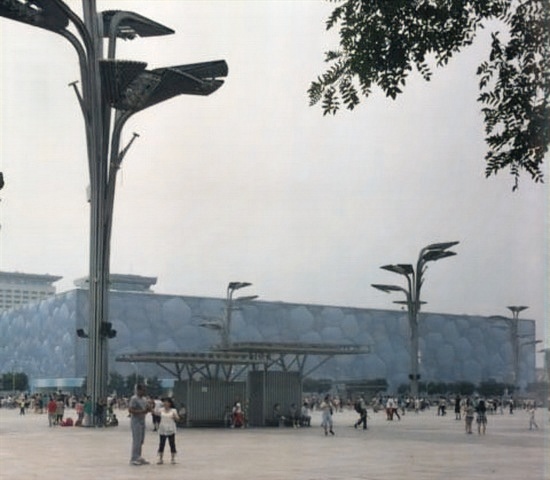} &
        \includegraphics[width=0.245\linewidth]{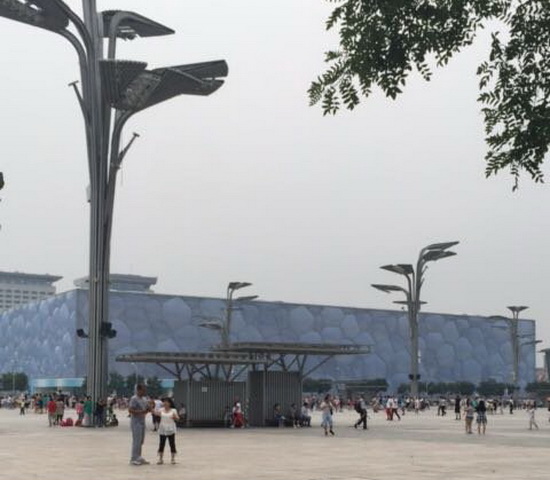} \\
        \includegraphics[width=0.245\linewidth]{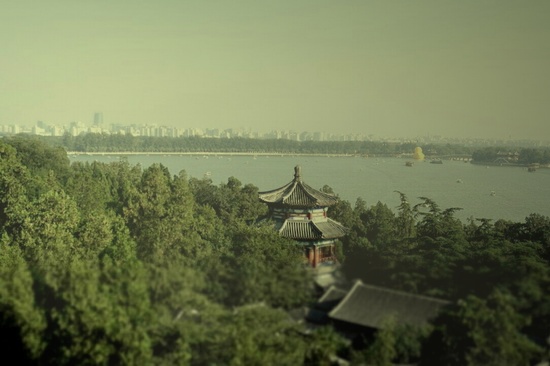} &
        \includegraphics[width=0.245\linewidth]{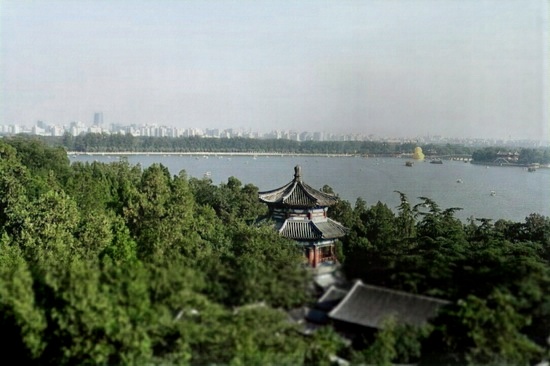} &
        \includegraphics[width=0.245\linewidth]{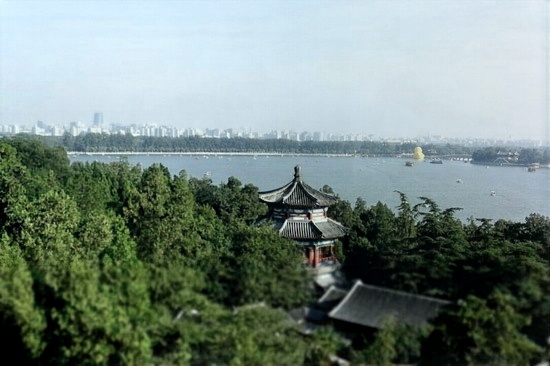} &
        \includegraphics[width=0.245\linewidth]{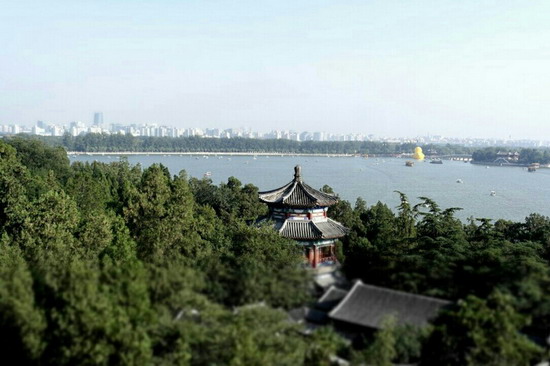} \\
        \addlinespace[2pt]
        Input & NAFNet & NAFNet+PAL & GT
    \end{tabular}
    \caption{Qualitative comparisons on nighttime dehazing (NAFNet on NHR). PAL reduces residual haze and color cast.}
    \label{fig:nightdehaze}
\end{figure}

Consequently, the ratio of gradient energies $\mathcal{E}_{\text{phot}} / \mathcal{E}_{\text{struct}}$ is proportional to $N/M$. Because $N$ is typically orders of magnitude larger than $M$, $\mathcal{E}_{\text{phot}}$ overwhelmingly overshadows $\mathcal{E}_{\text{struct}}$ (\textit{i.e.}, $\rho \to 1$), forcing the network to exhaust its gradient budget acting as a global color-matcher rather than a detail restorer. To validate this, we train a Retinexformer~\citep{cai2023retinexformer} on LOL-v1~\citep{wei2018deep} dataset and plot the val error decomposed into photometric and content in Figure~\ref{fig:gradient_stat} along with the photometric ratio $\rho$. It is clear that the photometric component is dominating the gradient, and the content improves slowly.

This motivates a loss function that explicitly removes the photometric component $\boldsymbol{\Delta}_{p}$ from the supervision signal, so that the full gradient budget is redirected toward structural restoration.

\subsection{Why Affine Alignment}

The gradient analysis above shows that the photometric component must be discounted from the loss function. This requires choosing an alignment model to estimate and remove this component. Alignment models can be ordered by expressiveness: a \emph{scalar} correction ($\alpha\mathbf{I}$, one parameter) removes brightness offset; a \emph{diagonal} model ($\text{diag}(\mathbf{d})\,\mathbf{I}$, three parameters) allows independent per-channel gain; and a \emph{full affine} model ($\mathbf{C}\mathbf{I}+\mathbf{b}$, twelve parameters) additionally captures cross-channel coupling and additive bias. Mean-brightness normalization~\citep{liao2025gtmean,zhang2019kindling} falls in the scalar family and can equalize overall luminance, yet it leaves color-temperature and white-balance shifts intact because these involve coupled, channel-dependent transformations. A diagonal model handles per-channel exposure differences but still cannot represent the off-diagonal terms. Figure~\ref{fig:photometric_analysis_grid} illustrates this on real LOL pairs. For each input, the optimal transform from each family is computed in closed form and applied. Only the full affine model reproduces the reference color, confirming that real photometric discrepancy requires cross-channel coupling to be modeled explicitly.

The affine model is the natural match for the nuisance we identified. The dominant photometric discrepancy across paired datasets is \emph{global}: it manifests as per-pair shifts in overall brightness, color temperature, and white balance, all of which are well described by a twelve-parameter affine transform. Crucially, fitting a global model to a per-image residual that also contains spatially localized content does not absorb that content. By construction, the least-squares affine fit captures only the variance that correlates globally with the prediction, while localized texture and structural differences remain in the residual and continue to supervise the network. Spatially varying photometric effects (e.g., vignetting and local illumination gradients) are not modeled by PAL; however, because they lack global correlation, the affine fit largely ignores them, and PAL turns toward standard pixel-wise supervision. Furthermore, when the photometric inconsistency is negligibly small, the regularized least-squares solution converges to $\mathbf{C}^{*}\!\to\!\mathbf{E}$, $\mathbf{b}^{*}\!\to\!\mathbf{0}$, so that $\mathcal{L}_{\text{PAL}}$ gracefully degenerates to the standard pixel-wise loss. Therefore, PAL discounts photometric nuisance when present and reduces to conventional supervision when absent. The performance across tasks with both global and localized degradation components (all weather) confirms this behavior empirically.

\begin{figure}[t]
    \centering
    \setlength{\tabcolsep}{1pt}
    \renewcommand{\arraystretch}{0.5}
    \begin{tabular}{cccc}
        \includegraphics[width=0.245\linewidth]{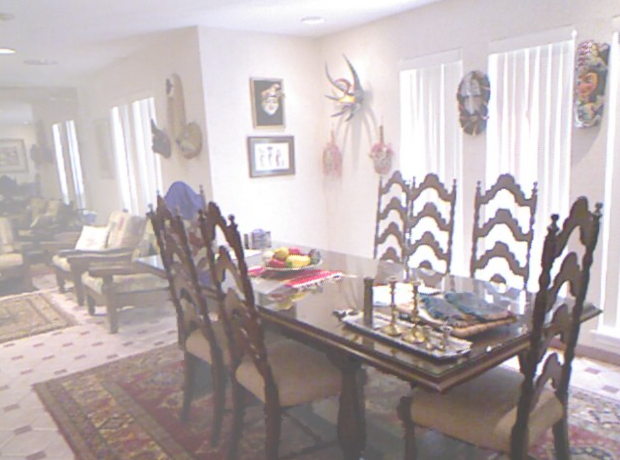} &
        \includegraphics[width=0.245\linewidth]{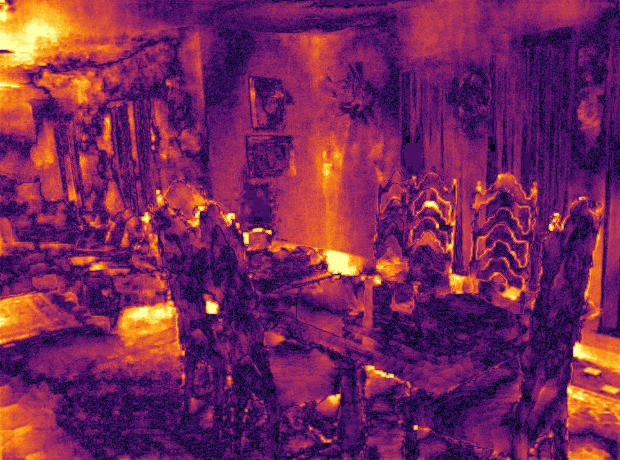} &
        \includegraphics[width=0.245\linewidth]{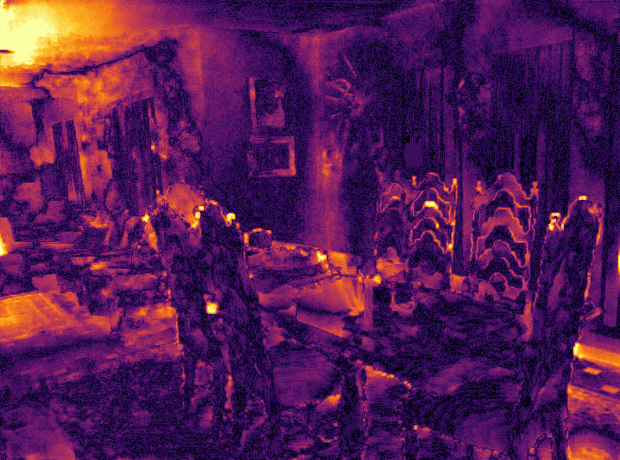} &
        \includegraphics[width=0.245\linewidth]{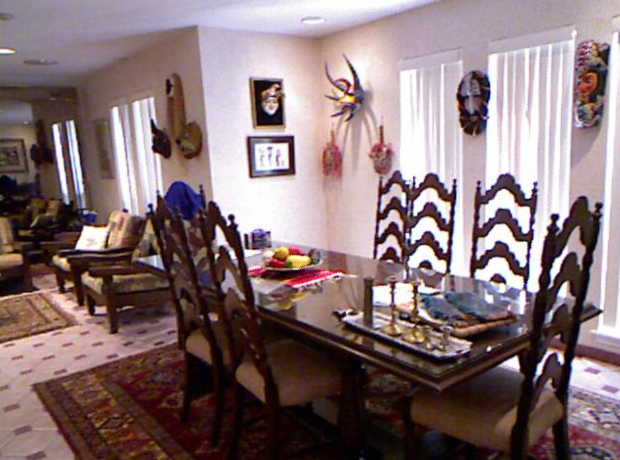} \\
        \includegraphics[width=0.245\linewidth]{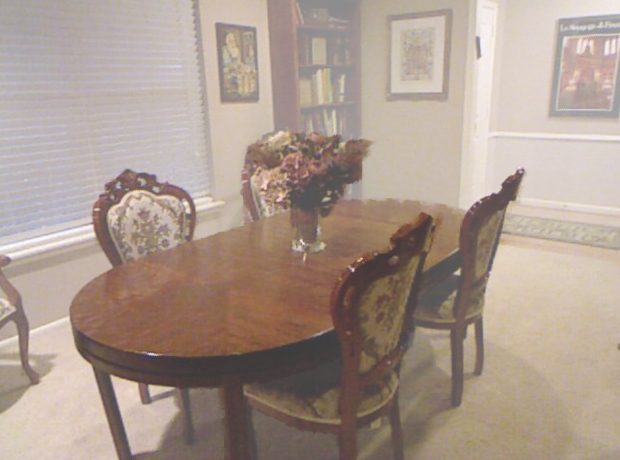} &
        \includegraphics[width=0.245\linewidth]{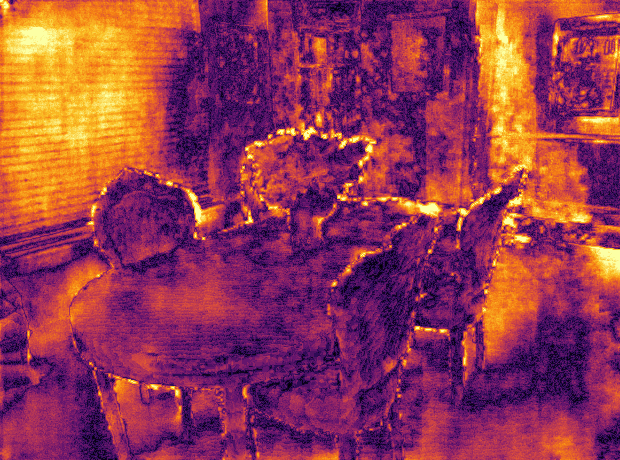} &
        \includegraphics[width=0.245\linewidth]{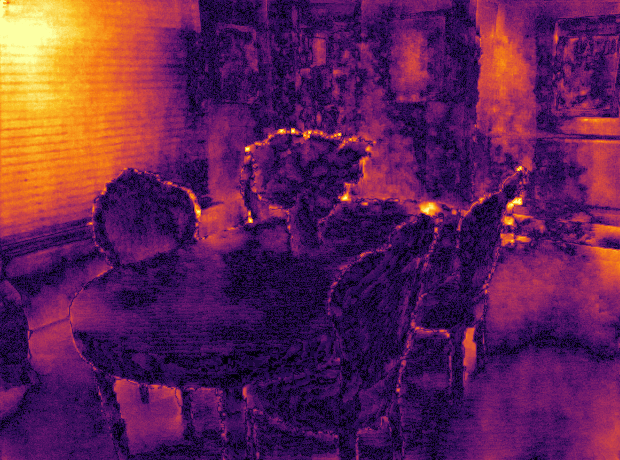} &
        \includegraphics[width=0.245\linewidth]{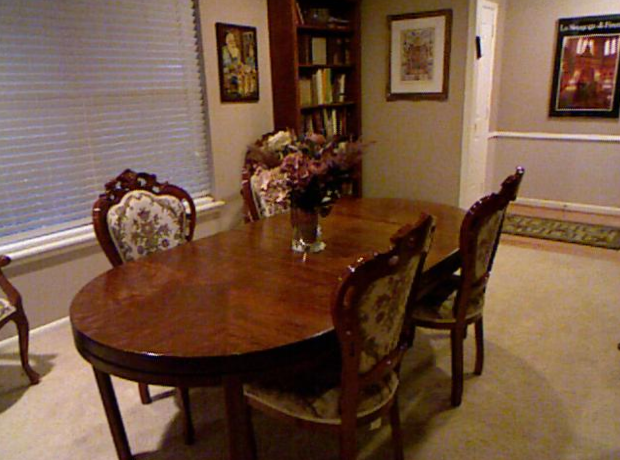} \\
        \addlinespace[2pt]
        Input & Baseline & +PAL (Ours) & GT
    \end{tabular}
    \caption{Qualitative comparison on MITNet image dehazing. We use error map to highlight the difference. Darker is better. PAL reduces residual color cast relative to the baseline.}
    \label{fig:dehaze}
\end{figure}

\begin{figure}[t]
    \centering
    \setlength{\tabcolsep}{1pt}
    \renewcommand{\arraystretch}{0.5}
    \begin{minipage}[t]{0.495\linewidth}
        \centering
        \begin{tabular}{ccc}
            \includegraphics[width=0.325\linewidth]{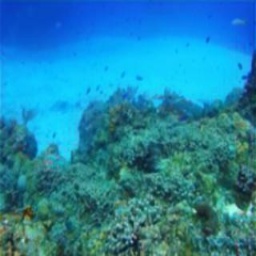} &
            \includegraphics[width=0.325\linewidth]{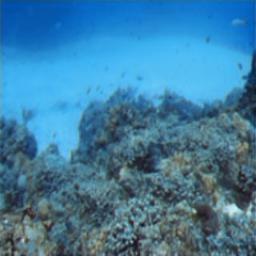} &
            \includegraphics[width=0.325\linewidth]{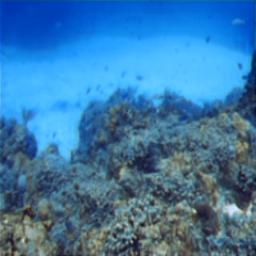} \\
            \includegraphics[width=0.325\linewidth]{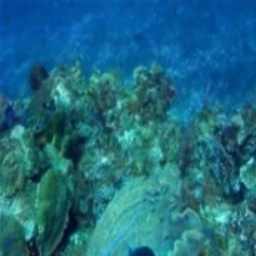} &
            \includegraphics[width=0.325\linewidth]{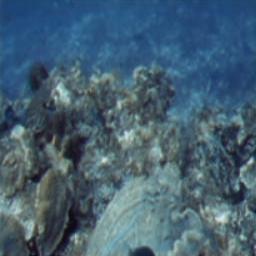} &
            \includegraphics[width=0.325\linewidth]{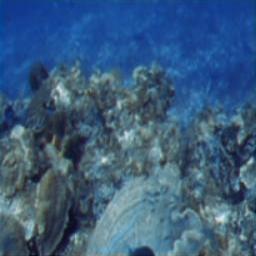} \\
            \addlinespace[2pt]
            Input & LiteEnhanceNet & +PAL (Ours)
        \end{tabular}
    \end{minipage}%
    \hfill
    \begin{minipage}[t]{0.495\linewidth}
        \centering
        \begin{tabular}{ccc}
            \includegraphics[width=0.325\linewidth]{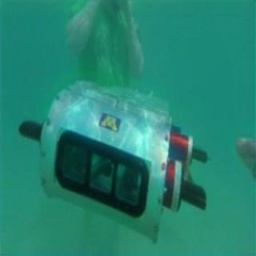} &
            \includegraphics[width=0.325\linewidth]{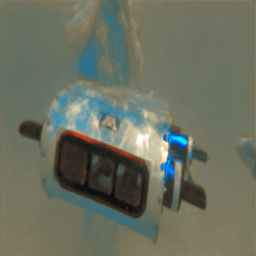} &
            \includegraphics[width=0.325\linewidth]{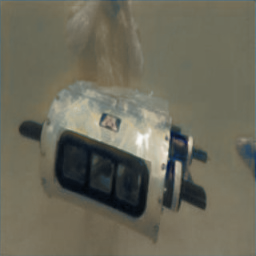} \\
            \includegraphics[width=0.325\linewidth]{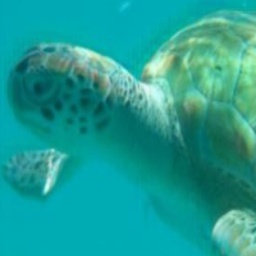} &
            \includegraphics[width=0.325\linewidth]{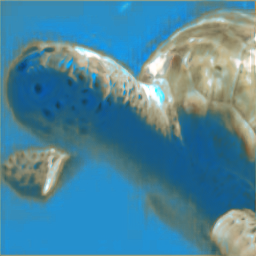} &
            \includegraphics[width=0.325\linewidth]{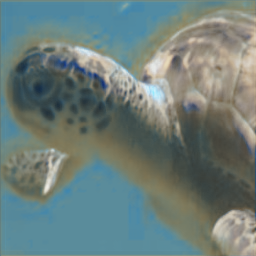} \\
            \addlinespace[2pt]
            Input & Boths & +PAL (Ours)
        \end{tabular}
    \end{minipage}
    \caption{Qualitative comparison on underwater image enhancement (EUVP). PAL produces outputs with more natural color and fewer artifacts.}
    \label{fig:uie}
\end{figure}

\begin{table*}[t]
\centering
\caption{Quantitative comparison on LOL, across PSNR ($\uparrow$), SSIM ($\uparrow$), LPIPS ($\downarrow$), IQA ($\uparrow$) and IAA ($\uparrow$).}
\setlength{\tabcolsep}{1pt}

\resizebox{\linewidth}{!}{
\begin{tabular}{l ccccc ccccc ccccc}
\toprule
\multirow{2}{*}{Methods} &
  \multicolumn{5}{c}{LOLv1} &
  \multicolumn{5}{c}{LOLv2-syn} &
  \multicolumn{5}{c}{LOLv2-real} \\
\cmidrule(lr){2-6} \cmidrule(lr){7-11} \cmidrule(lr){12-16}
 &
  PSNR & SSIM & LPIPS & IQA & IAA &
  PSNR & SSIM & LPIPS & IQA & IAA &
  PSNR & SSIM & LPIPS & IQA & IAA \\
\midrule
MIRNet~\citep{zamir2020learning}         & 20.57 & 0.769 & 0.254 & 2.914  & \textbf{1.855}  & 21.74 & 0.877 & 0.138 &  2.650  & 2.023 & 21.28 & 0.791 & 0.355 & 2.256 & 1.377 \\
+PAL (Ours)   & \textbf{21.01} & \textbf{0.791} & \textbf{0.228} & \textbf{2.923}   & 1.729  & \textbf{22.13} & \textbf{0.890} & \textbf{0.117} & \textbf{2.870}   & \textbf{2.091}  & \textbf{22.32} & \textbf{0.816} & \textbf{0.348} & \textbf{2.535}   & \textbf{1.469}  \\
\midrule
Uformer~\citep{wang2022uformer}        & 18.85 & 0.751 & 0.288 & 2.751 & 1.661  & 21.50 & 0.884 & 0.120 & 2.919 & 2.058  & 19.80 & 0.714 & 0.346 & 2.135  & 1.387 \\
+PAL (Ours)    & \textbf{19.31} & \textbf{0.767} & \textbf{0.247} & \textbf{2.858}   & \textbf{1.666} & \textbf{21.79} & \textbf{0.890} & \textbf{0.112} & \textbf{2.952}  & \textbf{2.063}  & \textbf{20.12} & \textbf{0.738} & \textbf{0.327} & \textbf{2.780}   & \textbf{1.649}  \\
\midrule
Retinexformer~\citep{cai2023retinexformer}  & 23.40 & 0.822 & 0.269 & 3.148  & 1.980  & 25.48 & 0.930 & 0.101 & 2.404  & 2.096  & 21.69 & 0.846 & 0.276 & 3.163   & 1.962  \\
+PAL (Ours)    & \textbf{24.53} & \textbf{0.847} & \textbf{0.239} & \textbf{3.533} & \textbf{2.122}   & \textbf{26.01} & \textbf{0.941} & \textbf{0.083} & \textbf{2.435}  & \textbf{2.162}  & \textbf{22.73} & \textbf{0.864} & \textbf{0.265} & \textbf{3.499}  & \textbf{2.077}  \\
\midrule
CID-Net~\citep{yan2025hvi}        & 23.97 & 0.849 & 0.104 & 3.791  & 2.071  & 25.44 & 0.935 & 0.047 & 3.299  & 2.171 & 23.19 & 0.857 & 0.136 & 3.699  & 2.042  \\
+PAL (Ours)    & \textbf{24.13}  & \textbf{0.854} & \textbf{0.099} & \textbf{3.923}  & \textbf{2.104}  & \textbf{25.84} & \textbf{0.937} & \textbf{0.045} & \textbf{3.373}  & \textbf{2.185}  & \textbf{23.95} & \textbf{0.870} & \textbf{0.112} & \textbf{3.938}  & \textbf{2.103}  \\
\bottomrule
\end{tabular}
}

\vspace{-10pt}
\label{tab:comparison}
\end{table*}

\begin{table}[t]
\centering
\caption{Underwater enhancement results on EUVP.}
\label{tab:underwater}
\setlength{\tabcolsep}{30pt}
\resizebox{\linewidth}{!}{
\begin{tabular}{llccc}
\toprule
\multirow{2}{*}{Method} & \multirow{2}{*}{Loss} & \multicolumn{3}{c}{EUVP} \\
\cmidrule(lr){3-5}
 & & PSNR$\uparrow$ & SSIM$\uparrow$ & LPIPS$\downarrow$ \\
\midrule
\multirow{2}{*}{Shallow-UWnet} & Baseline & 19.70 & 0.780 & 0.355 \\
 & +PAL & \textbf{20.35} & \textbf{0.790} & \textbf{0.327} \\
\midrule
\multirow{2}{*}{Boths} & Baseline & 19.68 & 0.748 & \textbf{0.367} \\
 & +PAL & \textbf{19.85} & \textbf{0.767} & 0.375 \\
\midrule
\multirow{2}{*}{LiteEnhanceNet} & Baseline & 20.40 & 0.779 & 0.343 \\
 & +PAL & \textbf{20.97} & \textbf{0.787} & \textbf{0.328} \\
\bottomrule
\end{tabular}
}
\end{table}

\begin{table}[t]
\centering
\caption{Dehazing results on RESIDE-SOTS-Indoor.}
\label{tab:dehaze}
\setlength{\tabcolsep}{30pt}
\resizebox{\linewidth}{!}{
\begin{tabular}{llccc}
\toprule
\multirow{2}{*}{Method} & \multirow{2}{*}{Loss} & \multicolumn{3}{c}{RESIDE-SOTS-Indoor} \\
\cmidrule(lr){3-5}
 & & PSNR$\uparrow$ & SSIM$\uparrow$ & LPIPS$\downarrow$ \\
\midrule
\multirow{2}{*}{FocalNet} & Baseline & 37.97 & 0.989 & \textbf{0.012} \\
 & +PAL & \textbf{38.18} & \textbf{0.991} & \textbf{0.012} \\
\midrule
\multirow{2}{*}{MITNet} & Baseline & 37.56 & 0.988 & 0.006 \\
 & +PAL & \textbf{37.80} & \textbf{0.988} & \textbf{0.005} \\
\midrule
\multirow{2}{*}{DehazeXL} & Baseline & 27.77 & 0.956 & 0.030 \\
 & +PAL & \textbf{28.07} & \textbf{0.961} & \textbf{0.028} \\
\bottomrule
\end{tabular}
}
\end{table}

\begin{table}[t]
\centering
\caption{Nighttime dehazing results on NHR.}
\label{tab:nightdehaze}
\setlength{\tabcolsep}{30pt}
\resizebox{\linewidth}{!}{
\begin{tabular}{llccc}
\toprule
\multirow{2}{*}{Method} & \multirow{2}{*}{Loss} & \multicolumn{3}{c}{NHR} \\
\cmidrule(lr){3-5}
 & & PSNR$\uparrow$ & SSIM$\uparrow$ & LPIPS$\downarrow$ \\
\midrule
\multirow{2}{*}{NAFNet} & Baseline & 22.06 & 0.825 & 0.082 \\
 & +PAL & \textbf{22.91} & \textbf{0.838} & \textbf{0.071} \\
\midrule
\multirow{2}{*}{Restormer} & Baseline & 18.30 & 0.794 & 0.118 \\
 & +PAL & \textbf{18.89} & \textbf{0.816} & \textbf{0.103} \\
\bottomrule
\end{tabular}
}
\end{table}

\begin{table}[t]
\centering
\caption{Shadow removal results on ISTD.}
\label{tab:shadow}
\setlength{\tabcolsep}{50pt}
\resizebox{\linewidth}{!}{
\begin{tabular}{llccc}
\toprule
\multirow{2}{*}{Method} & \multirow{2}{*}{Loss} & \multicolumn{3}{c}{ISTD} \\
\cmidrule(lr){3-5}
 & & PSNR$\uparrow$ & SSIM$\uparrow$ & RMSE$\downarrow$ \\
\midrule
\multirow{2}{*}{RASM} & Baseline & 32.32 & \textbf{0.968} & \textbf{4.12} \\
 & +PAL & \textbf{32.65} & \textbf{0.968} & 4.16 \\
\midrule
\multirow{2}{*}{HomoFormer} & Baseline & 32.02 & \textbf{0.968} & 4.24 \\
 & +PAL & \textbf{32.49} & \textbf{0.968} & \textbf{4.17} \\
\bottomrule
\end{tabular}
}
\end{table}
\subsection{Photometric Alignment Loss (PAL)}
We model the photometric discrepancy between prediction and target as a global affine color transform, defined by a $3\times3$ matrix $\mathbf{C}$ and a $3\times1$ bias vector $\mathbf{b}$:
\begin{equation}
    \mathbf{I}_{\text{gt}} \approx \mathbf{C} \hat{\mathbf{I}} + \mathbf{b}.
\end{equation}
This model captures per-channel gains, cross-channel coupling, and additive color shifts. PAL computes the least-squares alignment that best explains this discrepancy, then measures the residual reconstruction error after alignment. In this way, PAL preserves supervision for content while reducing the influence of photometric mismatch that would dominate or corrupt pixel-wise training.

Our goal is to find the optimal parameters $(\mathbf{C^*}, \mathbf{b^*})$ that minimize the expected squared L2-norm of the residual:
\begin{equation}
    \mathcal{L}(\mathbf{C}, \mathbf{b}) = \mathbb{E}\left[ \|(\mathbf{C} \hat{\mathbf{I}} + \mathbf{b}) - \mathbf{I}_{\text{gt}}\|_2^2 \right].
\end{equation}
The standard solution from multivariate linear regression is $\mathbf{b^*} = \mu_\text{gt} - \mathbf{C^*} \mu_{\hat{\mathbf{I}}}$ and $\mathbf{C^*} = \text{Cov}(\mathbf{I}_{\text{gt}},
\hat{\mathbf{I}}) \text{Cov}(\hat{\mathbf{I}}, \hat{\mathbf{I}})^{-1}$. However, a practical issue arises when the prediction has low color variance (e.g., large monochromatic regions). In such cases, the covariance matrix $\text{Var}(\hat{\mathbf{I}})$ can become ill-conditioned or singular, making its inverse numerically unstable and leading to extreme values in $\mathbf{C^*}$.
To guarantee a stable solution, we incorporate Ridge Regression by adding an L2 regularization term. Consequently, the solution for the desired transformation matrix $\mathbf{C^*}$ becomes the following:
\begin{equation}
    \mathbf{C^*} = \text{Cov}(\mathbf{I}_{\text{gt}}, \hat{\mathbf{I}}) \left( \text{Cov}(\hat{\mathbf{I}}, \hat{\mathbf{I}}) + \epsilon \mathbf{E} \right)^{-1}.
    \label{eq:c_star_ridge}
\end{equation}
where $\epsilon$ is a small, positive hyperparameter that controls the regularization strength, and $\mathbf{E}$ is the $3\times3$ identity matrix. This regularization term ensures that the matrix to be inverted is always well-conditioned. The optimal bias remains:
\begin{equation}
    \mathbf{b^*} = \mu_{\text{gt}} - \mathbf{C^*} \mu_{\hat{\mathbf{I}}},
    \label{eq:b_star_ridge}
\end{equation}
where $\mu_{\hat{\mathbf{I}}} = \mathbb{E}[\hat{\mathbf{I}}]$ and $\mu_{\text{gt}} = \mathbb{E}[\mathbf{I}_{\text{gt}}]$. The covariance matrices are defined as $\text{Cov}(\hat{\mathbf{I}}, \hat{\mathbf{I}}) = \mathbb{E}[(\hat{\mathbf{I}} - \mu_{\hat{\mathbf{I}}})(\hat{\mathbf{I}} - \mu_{\hat{\mathbf{I}}})^\top]$ and $\text{Cov}(\mathbf{I}_{\text{gt}}, \hat{\mathbf{I}}) = \mathbb{E}[(\mathbf{I}_{\text{gt}} - \mu_{\text{gt}})(\hat{\mathbf{I}} - \mu_{\hat{\mathbf{I}}})^\top]$. All required statistics can be computed efficiently over training samples.

\begin{table*}[t]
\centering
\caption{All-in-one restoration results across multiple weather degradation benchmarks. }
\setlength{\tabcolsep}{2.9pt}
\resizebox{\linewidth}{!}{
\begin{tabular}{ll ccc ccc ccc ccc}
\toprule
\multirow{2}{*}{Method} & \multirow{2}{*}{Loss}
 & \multicolumn{3}{c}{Snow100K-S}
 & \multicolumn{3}{c}{Snow100K-L}
 & \multicolumn{3}{c}{Outdoor}
 & \multicolumn{3}{c}{RainDrop} \\
\cmidrule(lr){3-5} \cmidrule(lr){6-8} \cmidrule(lr){9-11} \cmidrule(lr){12-14}
 & & PSNR$\uparrow$ & SSIM$\uparrow$ & LPIPS$\downarrow$
   & PSNR$\uparrow$ & SSIM$\uparrow$ & LPIPS$\downarrow$
   & PSNR$\uparrow$ & SSIM$\uparrow$ & LPIPS$\downarrow$
   & PSNR$\uparrow$ & SSIM$\uparrow$ & LPIPS$\downarrow$ \\
\midrule
\multirow{2}{*}{Histoformer~\citep{sun2024restoring}}
 & Baseline & 37.41 & 0.965 & 0.045 & 32.16 & 0.926& 0.0919 & 32.08 & 0.938 & 0.077 & \textbf{33.06} & 0.944 &  0.067 \\
 & +PAL & \textbf{37.85} & \textbf{0.968} &  \textbf{0.039} & \textbf{32.34} &\textbf{0.929}  & \textbf{0.087} & \textbf{32.82}  & \textbf{0.945} & \textbf{0.071}  & 32.77 & \textbf{0.945} & \textbf{0.063} \\ 
\midrule
\multirow{2}{*}{MODEM~\citep{wang2025modem}}
 & Baseline & 38.08 & 0.967 & 0.041 & 32.52 & 0.929 & 0.088  & 33.10 & 0.941 & 0.070 & 33.01 & 0.943 &  0.065 \\
 & +PAL &  \textbf{38.10}&  \textbf{0.968} & \textbf{0.039} & \textbf{32.54} & \textbf{0.930} & \textbf{0.085}  & \textbf{33.25} & \textbf{0.942} & \textbf{0.065} & \textbf{33.08}  & \textbf{0.944} &  \textbf{0.060}  \\
\bottomrule
\end{tabular}
}
\vspace{-12pt}
\label{tab:allinone}
\end{table*}

\noindent\textbf{Integration into training.}
With the numerically stable optimal transformation $(\mathbf{C^*}, \mathbf{b^*})$, we define our Photometric Alignment Loss (PAL) as the minimum reconstruction error:
\begin{equation}
    \mathcal{L}_{\text{PAL}} = \|(\mathbf{C^*} \hat{\mathbf{I}} + \mathbf{b^*}) - \mathbf{I}_{\text{gt}}\|_.
\end{equation}
During training, it can be integrated with the existing loss $\mathcal{L}_{\text{pixel}}$:
\begin{equation}
    \mathcal{L}_{\text{total}} = \mathcal{L}_{\text{pixel}} + \alpha \mathcal{L}_{\text{PAL}}.
\end{equation}
Retaining $\mathcal{L}_{\text{pixel}}$ alongside $\mathcal{L}_{\text{PAL}}$ is deliberate: $\mathcal{L}_{\text{pixel}}$ preserves full pixel-level fidelity supervision, while $\mathcal{L}_{\text{PAL}}$ supplies a photometric-invariant gradient that emphasizes content restoration. 
Here, $\mathbf{C^*}$ and $\mathbf{b^*}$ are computed on-the-fly and then treated as constants (stop-gradient) for the backward pass; this is essential because, if gradients were allowed to flow through $\mathbf{C^*}$ and $\mathbf{b^*}$, the network could trivially minimize $\mathcal{L}_{\text{PAL}}$ without improving structural content, collapsing to degenerate solutions. $\alpha$ is a scalar hyperparameter that balances the pixel term and the alignment term. The compute only costs 0.0037 GFLOPs on a $256\times256$ image, on the order of 0.01\%–0.1\% of the backbone. PAL is therefore easy to integrate into existing paired low-level vision pipelines.

\begin{figure}[t]
    \centering
    \setlength{\tabcolsep}{0.5pt}
    \renewcommand{\arraystretch}{0.1}
    \begin{tabular}{cccc}
        \begin{subfigure}{0.245\linewidth}
            \centering
            \includegraphics[width=\linewidth]{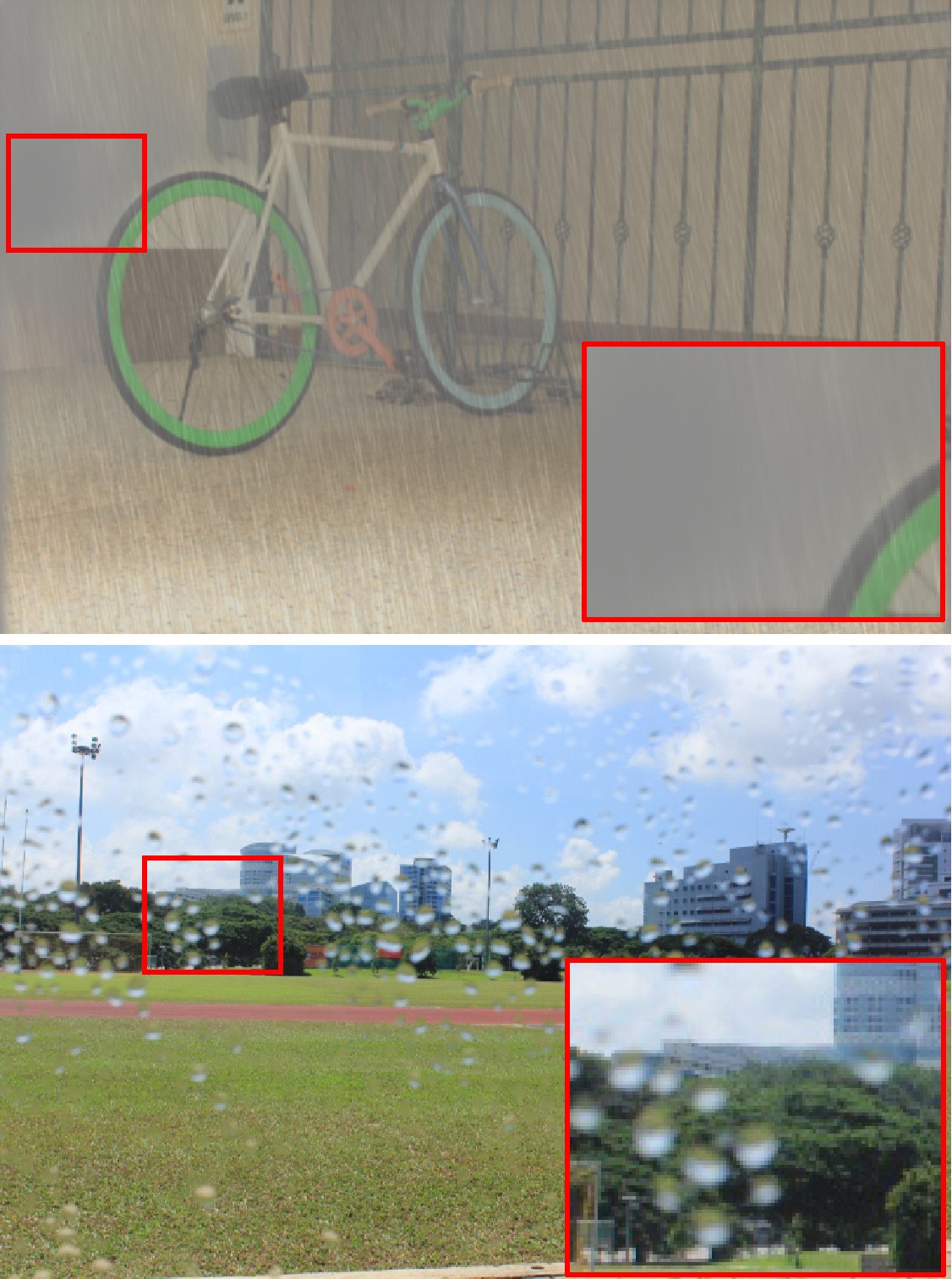}
        \end{subfigure}
        &
        \begin{subfigure}{0.245\linewidth}
            \centering
            \includegraphics[width=\linewidth]{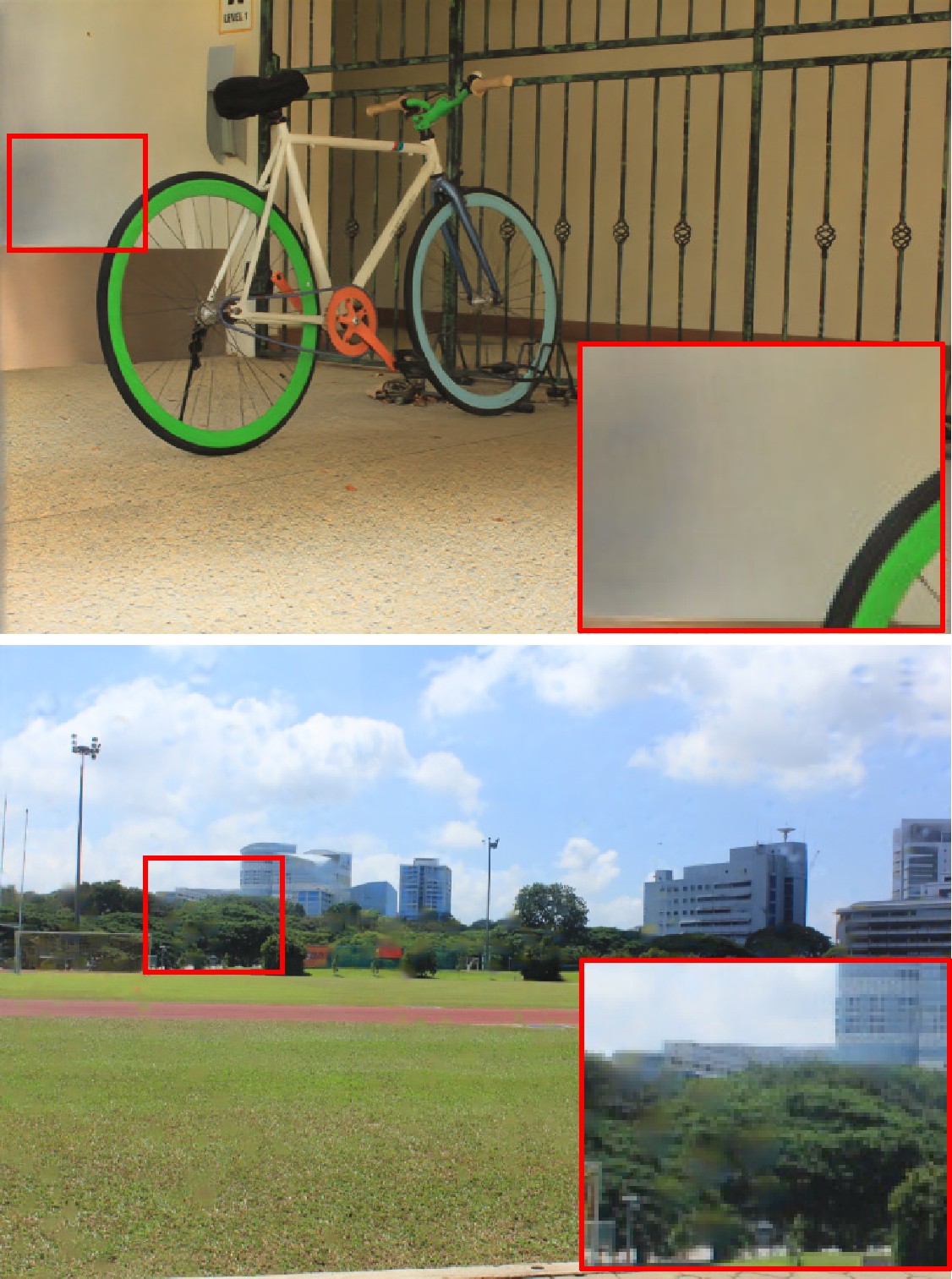}
        \end{subfigure}
        &
        \begin{subfigure}{0.245\linewidth}
            \centering
            \includegraphics[width=\linewidth]{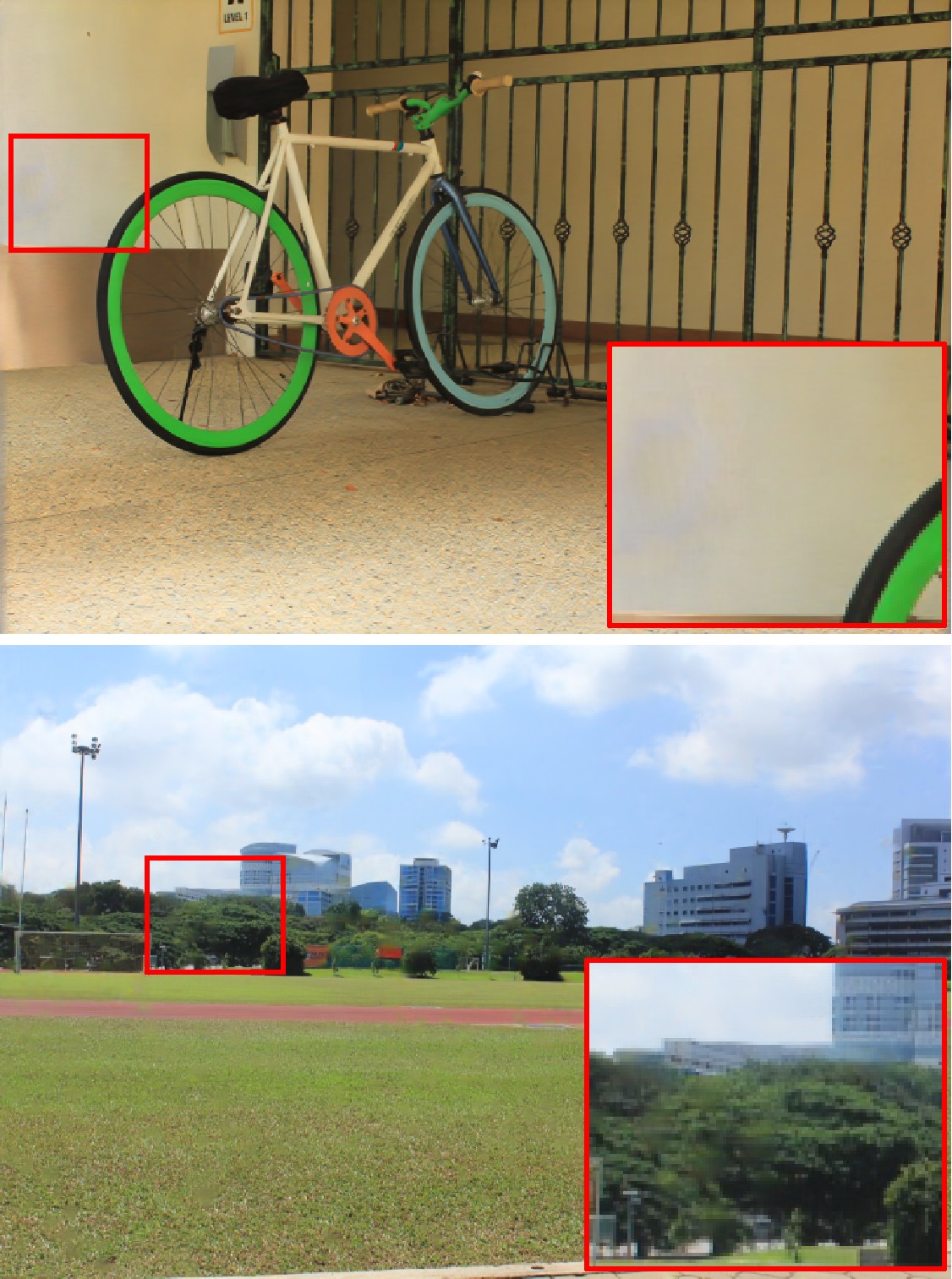}
        \end{subfigure}
        &
        \begin{subfigure}{0.245\linewidth}
            \centering
            \includegraphics[width=\linewidth]{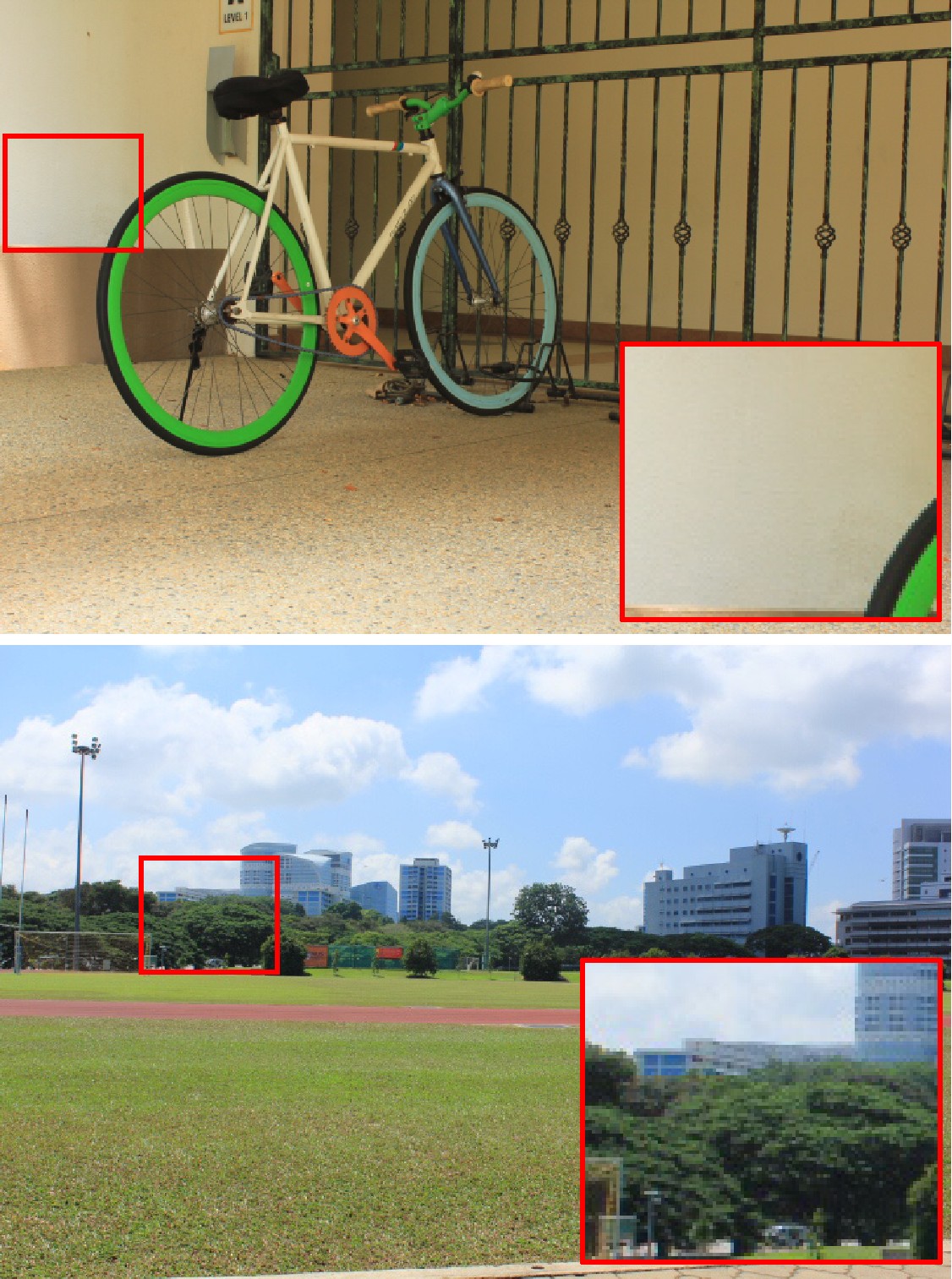}
        \end{subfigure}
        \\
        \addlinespace[2pt]
        Input & Histoformer & Histo.+PAL & GT \\
    \end{tabular}
    \caption{Qualitative comparison on the all-in-one task.}
    \vspace{-5pt}
    \label{fig:allinone_vis}
\end{figure}

\section{Experimental Validation}

We organize the experiments by task type. We first evaluate PAL on tasks where photometric inconsistency is intrinsic to the task, including low-light and underwater enhancement (Section~4.1), and then on tasks where it is induced by data acquisition, including dehazing, nighttime dehazing, and all-in-one restoration (Section~4.2). We further examine shadow removal as a hybrid extension where both sources coexist (Section~4.3), followed by ablation studies (Section~4.4). For all experiments, we keep the original settings of each baseline unchanged and introduce PAL as the only modification.

\subsection{Tasks with Intrinsic Photometric Transfer}

Our evaluation of the task-intrinsic source comes from two representative enhancement tasks: low-light image enhancement (LLIE) and underwater image enhancement (UIE). For the low-light task, we evaluate our method on the LOLv1~\citep{wei2018deep}, LOLv2-real~\citep{yang2021sparse}, and LOLv2-synthetic~\citep{yang2021sparse} datasets, using four backbone architectures spanning different design philosophies: MIRNet~\citep{zamir2020learning} (multi-scale residual), Uformer~\citep{wang2022uformer} (window-based transformer), Retinexformer~\citep{cai2023retinexformer} (Retinex-guided transformer), and HVI-CIDNet~\citep{yan2025hvi} (learnable color space). For the underwater task, we conduct experiments on the EUVP~\citep{islam2020fast} dataset and employ three task-specific backbones: Shallow-UWNet~\citep{naik2021shallow}, Boths~\citep{liu2022boths}, and LiteEnhanceNet~\citep{zhang2024liteenhancenet}. For quantitative evaluation across all paired datasets, we report PSNR, SSIM~\citep{wang2004image}, and LPIPS~\citep{zhang2018unreasonable}. We maintain the original training configurations for all backbones and integrate PAL as the only modification.

\noindent\textbf{Quantitative results.}
Table~\ref{tab:comparison} shows that PAL consistently improves all four backbones across the three LOL benchmarks on all five metrics (PSNR, SSIM, LPIPS, IQA, IAA). Notably, the improvements are not limited to photometric fidelity (PSNR): structural quality (SSIM, LPIPS) also improves, confirming that discounting the nuisance photometric component re-allocates gradient budget to restoration-relevant content. The gains are most pronounced for Retinexformer (+1.13~dB on LOLv1, +1.04~dB on LOLv2-real), which uses a Retinex decomposition that is particularly sensitive to photometric ambiguity. Table~\ref{tab:underwater} reports UIE results on EUVP. PAL improves all three backbones, with LiteEnhanceNet gaining +0.57~dB PSNR and Shallow-UWNet gaining +0.65~dB. This confirm that PAL addresses a general photometric inconsistency phenomenon rather than a dataset-specific artifact.

\noindent\textbf{Qualitative results.} We present visual comparisons on low-light enhancement in Figure~\ref{fig:llie_nightdehaze} and nighttime dehazing in Figure~\ref{fig:nightdehaze}. PAL produces results with less noise and a more pleasant color. As in Figure~\ref{fig:uie}, our UIE result is natural and free from artifact while the baseline method shows an unpleasant shift.

\subsection{Tasks with Acquisition-Induced Mismatch}
Dehazing, nighttime dehazing, and all-in-one weather restoration are representative of acquisition-induced mismatch. Here, the ground truth should ideally share the same photometric profile as the clean scene content, yet data collection under varying conditions introduces per-pair photometric shifts from differing sensor responses, lighting, and environmental scattering. For image dehazing, we evaluate on RESIDE-SOTS-Indoor~\citep{li2018benchmarking} using FocalNet~\citep{cui2023focal}, MITNet~\citep{shen2023mutual}, and DehazeXL~\citep{chen2025tokenize}. For nighttime dehazing, we evaluate on NHR~\citep{zhang2020nighttime} with NAFNet~\citep{chen2022simple} and Restormer~\citep{zamir2022restormer}. For all-weather restoration, we evaluate on Snow100K-S, Snow100K-L, Outdoor, and RainDrop using Histoformer~\citep{sun2024restoring} and MODEM~\citep{wang2025modem}. 

\noindent\textbf{Quantitative results.}
Table~\ref{tab:dehaze} shows that PAL improves all three dehazing backbones on RESIDE-SOTS-Indoor. 
Table~\ref{tab:nightdehaze} reports nighttime dehazing results on NHR. The improvements are substantial. NAFNet gains +0.85~dB PSNR and Restormer gains +0.59~dB, with corresponding LPIPS improvements. Nighttime conditions amplify photometric inconsistency through spatially non-uniform artificial lighting and color-shifted scattering, making PAL's explicit alignment beneficial. Table~\ref{tab:allinone} presents all-weather restoration results. PAL improves both two models across most benchmarks. This is notable because they is trained on data pooled from multiple degradation, each collected under different imaging conditions with its own photometric profile. The inter-dataset photometric inconsistency compounds the per-pair inconsistency, yet PAL handles both.

\noindent\textbf{Qualitative results.}
Figure~\ref{fig:dehaze} presents visual comparisons on dehazing, where PAL reduces residual color cast. Figure~\ref{fig:allinone_vis} shows an all-in-one restoration example on deraining, where PAL produces cleaner outputs with fewer color artifacts.

\begin{table*}[t]
\centering
\setlength{\tabcolsep}{9.7pt}
\resizebox{\linewidth}{!}{
\begin{tabular}{lccccccccccc}
\toprule
\multirow{3}{*}{Setup} & \multicolumn{6}{c}{The choice of $\alpha$.} & \multicolumn{5}{c}{The choice of $\epsilon$.} \\
\cmidrule(r){2-7} \cmidrule(l){8-12}
        & 0.1 & 0.5 & 0.6 & 0.8 & 1.0 & 1.2 & 0.0001 & 0.001 & 0.01 & 0.1 & 1    \\
\midrule
PSNR $\uparrow$                 & 23.07 & 23.37 & 23.95 & 23.41 & 23.19  & 23.45 & NaN & 23.41 & 23.08 & 23.17 & 23.12 \\
SSIM $\uparrow$                 & 0.822 & 0.840 & 0.870  & 0.861 & 0.857 & 0.831 & NaN & 0.861 & 0.842 & 0.850 & 0.834 \\
\bottomrule
\end{tabular}
}
\caption{Ablation analysis of the weight $\alpha$ for our PAL and the regularization term $\epsilon$ for matrix inversion.}
\vspace{-5pt}
\label{tab:ablation}
\end{table*}

\begin{table*}[t]
\centering
\setlength{\tabcolsep}{9pt}
\resizebox{\linewidth}{!}{
\begin{tabular}{l*{5}{cc}}
\toprule
\multirow{2}{*}{Method} &
\multicolumn{2}{c}{DICM} &
\multicolumn{2}{c}{LIME} &
\multicolumn{2}{c}{MEF} &
\multicolumn{2}{c}{NPE} &
\multicolumn{2}{c}{VV} \\
& IQA$\uparrow$ & IAA$\uparrow$
& IQA$\uparrow$ & IAA$\uparrow$
& IQA$\uparrow$ & IAA$\uparrow$
& IQA$\uparrow$ & IAA$\uparrow$
& IQA$\uparrow$ & IAA$\uparrow$ \\
\midrule
MIRNet~\citep{zamir2020learning}   & 2.680   & 1.892  & 2.409  & 1.793  & 2.174  & 1.557  & 2.193  & 1.710 & 2.580 & 1.669 \\
+PAL (Ours) &\textbf{2.875}  & \textbf{2.034} &\textbf{2.633}  & \textbf{1.939} & \textbf{2.404} & \textbf{1.705} &\textbf{2.496}  &\textbf{1.854}  & \textbf{2.680} & \textbf{1.740} \\
\midrule
Uformer~\citep{wang2022uformer}& 2.548  & 1.852 & 2.499  & 1.986  & 2.308 & 2.017 & 2.435 & 1.832 & 2.567 & 1.710  \\
+PAL (Ours)&\textbf{2.640}  & \textbf{1.870} & \textbf{2.604} &\textbf{2.026}  & \textbf{2.341} &\textbf{2.080} &\textbf{2.699}  &\textbf{1.875} &\textbf{2.612}  & \textbf{1.712} \\
\midrule
Retinexformer~\citep{cai2023retinexformer} & 2.869 & 2.321  & 3.005  & 2.528 & 3.193 & 2.520  & 2.785  & 2.175 & 2.464 & 1.923  \\
+PAL (Ours)       & \textbf{2.907} &\textbf{2.391}  & \textbf{3.243} &\textbf{2.668}  &\textbf{3.343} & \textbf{2.586} &\textbf{2.877}  & \textbf{2.227} &\textbf{2.668}  &\textbf{2.039}  \\
\midrule
CID-Net~\citep{yan2025hvi} & 3.464 & 2.340   & 3.286   & 2.439  & 3.497  & 2.482   & 3.061  & 2.064   & 3.086 & 1.927    \\
+PAL (Ours)       & \textbf{3.632} &\textbf{2.449}  & \textbf{3.433} &\textbf{2.549}  & \textbf{3.474} & \textbf{2.448} &\textbf{3.434}  & \textbf{2.339} &\textbf{3.174}   &\textbf{1.995}  \\
\bottomrule
\end{tabular}
}

\caption{Quantitative results on unpaired datasets across IQA ($\uparrow$) and IAA ($\uparrow$), evaluated using Q-Align.}
\vspace{-10pt}
\label{tab:gen}
\end{table*}

\subsection{Hybrid Case: Shadow Removal}

Shadow removal provides a case in which both sources of per-pair photometric inconsistency coexist within a single image. Inside shadow regions, the model must learn to undo the illumination change. Outside shadow regions, residual photometric deviation is acquisition-induced, as the paired shadow and shadow-free images are captured under slightly different conditions. Because these two regions undergo fundamentally different photometric shifts, a single global affine fit would conflate them. We therefore extend the PAL to a masked version, since the shadow mask is already a standard input to existing pipelines~\citep{mei2024latent, GuoHLCW23}.  We calculate the photometric matrix inside and outside the mask as a spatial partition. We evaluate on the ISTD~\citep{WangL018} dataset with RASM~\citep{liu2025rasm} and HomoFormer~\citep{Xiao_2024_CVPR}. In Table~\ref{tab:shadow}, PAL improves PSNR for both methods (+0.33~dB for RASM, +0.47~dB for HomoFormer) with comparable SSIM and RMSE. 

\subsection{Ablations and Discussion}

We conduct ablation studies on the LOLv2-real dataset using HVI-CIDNet as the backbone to analyze the impact of the two key hyperparameters $\alpha$ and $\epsilon$. We further examine cross-dataset generalization on unseen unpaired low-light datasets to understand whether PAL improves robustness beyond the training distribution.

\noindent\textbf{Effect of Weight $\alpha$.}
We first study the influence of the weighting factor $\alpha$ for PAL, with results shown in Table~\ref{tab:ablation}. As $\alpha$ increases from 0.1 to 0.6, performance steadily improves, indicating that the alignment term provides a useful complementary signal to the pixel-wise loss. The best performance is achieved at $\alpha = 0.6$, with a PSNR of 23.95 dB and an SSIM of 0.870. When $\alpha$ is increased further ($\alpha \ge 0.8$), performance begins to decline, suggesting that over-emphasizing alignment can interfere with the learning of fine-grained restoration details. We thus set $\alpha$ to 0.6 for enhancement tasks, while we empirically set $\alpha$ to 0.8 for the restoration task to further discount photometric discrepancy that shouldn't be there.

\noindent\textbf{Effect of Regularization Term $\epsilon$.}
Next, we analyze the regularization term $\epsilon$ (Table~\ref{tab:ablation}). Setting $\epsilon$ to a very small value (0.0001) resulted in `NaN` losses during training, confirming that regularization is necessary when the input has low color variance, especially early in training. As $\epsilon$ increases, performance degrades because a larger $\epsilon$ biases the transformation matrix $C^*$ toward a scaled identity and reduces its ability to model color correlations. Our experiments show that $\epsilon = 0.001$ provides the best trade-off, so we apply it to all of our experiments. Please note that the images are in [0, 1].

\noindent\textbf{Cross-dataset generalization.}
To assess whether PAL also improves generalization rather than overfitting, we evaluate LOL-trained models on unseen low-light datasets: DICM~\citep{lee2013contrast}, LIME~\citep{DBLP:journals/tip/GuoLL17}, MEF~\citep{ma2015perceptual}, NPE~\citep{wang2013naturalness}, and VV~\citep{vonikakis2018evaluation}. Since they do not provide paired ground truth, we report non-reference quality assessment (IQA) and aesthetic assessment (IAA) scores computed by Q-Align~\citep{wu2023q}, following recent works~\citep{yan2025hvi}. As in Table~\ref{tab:gen}, PAL consistently improves both IQA and IAA across all 4 backbones and datasets. This indicates that PAL reduces overfitting to the photometric profile of the training set and leads to outputs with more natural color and better perceptual quality on out-of-distribution data.

\section{Conclusion}

Paired low-level vision tasks suffer from per-pair photometric inconsistency: different image pairs demand different global photometric mappings, whether because photometric transfer is intrinsic to the task or because data acquisition introduces unintended shifts. We showed that this produces a unified optimization pathology in which standard pixel-wise losses allocate disproportionate gradient budget to conflicting photometric targets, with severity determined by the magnitude of inconsistency and the dataset size. To address this, we proposed PAL, which models photometric discrepancy with a closed-form color alignment before measuring reconstruction residuals. PAL is flexible, computationally negligible, and easy to integrate into existing pipelines. Across experiments covering 16 datasets, 6 tasks, and 16 methods on enhancement, restoration, and hybrid tasks, PAL consistently improves fidelity metrics and cross-dataset generalization. These findings highlight the importance of explicitly accounting for photometric inconsistency in paired supervision and suggest a promising direction for designing more robust objectives in low-level vision.

\section*{Acknowledgement}
The authors would like to express their gratitude to TPU Research Cloud (TRC) for computational resources. 

\bibliography{main}
\bibliographystyle{tmlr}

\clearpage

\appendix

\section{Limitations and Future Work}

PAL models the photometric discrepancy as a global affine color transformation ($\mathbf{C}\hat{\mathbf{I}}+\mathbf{b}$, 12 parameters), which cannot explicitly capture spatially varying photometric effects such as local illumination. However, this is a deliberate design choice: a global model avoids absorbing spatially localized content (textures, edges) into the alignment, which would undermine restoration supervision. Our all-weather restoration experiments (Table~6 of the main paper), which include both global and localized degradations, empirically confirm that PAL does not interfere with localized restoration. A promising future direction is to explore patch-wise or spatially adaptive affine partitions that can handle local photometric variation while retaining the closed-form efficiency.

Real camera pipelines involve non-linear operations such as gamma correction and tone mapping. PAL's affine model provides a first-order approximation to these transformations. While this is a simplification, it is well justified: in the neighborhood of the operating point, most smooth non-linear color transforms are well approximated by their local tangent (affine) map. Moreover, the affine model strikes an effective balance between expressiveness and robustness, as it captures the dominant modes of photometric variation (gain, bias, cross-channel coupling) without risking overfitting to image content. Our extensive experiments across 16 datasets with diverse imaging pipelines demonstrate that this approximation is practically sufficient. Extending PAL to higher-order models (e.g., polynomial color transforms) is a natural future direction, though care must be taken to prevent the alignment from absorbing content-relevant signal.

\section{Scope and Applicability}

We discuss the scope of tasks and scenarios where PAL is expected to be most beneficial, as well as cases where its impact is limited.

\noindent\textbf{Tasks with significant per-pair photometric inconsistency.}
PAL provides the largest improvements when the training data exhibits substantial per-pair variation in global brightness, color, or white balance. This includes: (1)~enhancement tasks where photometric transfer is intrinsic (low-light enhancement, underwater image enhancement), (2)~restoration tasks where acquisition mismatch introduces spurious photometric shifts (dehazing, deraining), and (3)~multi-dataset training (all-in-one restoration) where different constituent datasets have distinct photometric profiles. In all these cases, the per-pair photometric component dominates the gradient energy (high $\rho$ in Eq.~(4) of the main paper), and PAL effectively redirects the gradient budget toward structural content.

\noindent\textbf{Tasks with minimal photometric shift.}
For tasks such as image super-resolution and Gaussian denoising, the ground truth and input share nearly identical photometric profiles by construction. In these cases, the affine alignment converges to identity, and PAL provides marginal or no improvement because there is little photometric nuisance to discount.  PAL operates in the RGB pixel space and modifies only the reconstruction loss. It is therefore fully compatible with, and complementary to, perceptual losses~\citep{johnson2016perceptual}, adversarial losses~\citep{goodfellow2014generative}, and frequency-domain losses. PAL removes the photometric nuisance from the pixel-level supervision, while these complementary objectives provide additional constraints on perceptual quality or texture fidelity. 
\section{Discussion: PAL vs. GT-Mean}

GT-Mean~\citep{liao2025gtmean,zhang2019kindling} is an alignment technique tailored for low-light enhancement that aligns the global brightness of the prediction to the ground truth via a single scalar ratio before computing the pixel-wise loss or metrics. In this section, we provide a self-contained, formal comparison between PAL and GT-Mean to clarify their relationship in full detail.

\subsection{Explicit Formulations}

We first state both formulations explicitly. Let $\hat{\mathbf{I}} \in \mathbb{R}^{3 \times N}$ denote the predicted image and $\mathbf{I}_{\text{gt}} \in \mathbb{R}^{3 \times N}$ denote the ground truth, where $N = H \times W$ is the number of pixels. Each column $\hat{\mathbf{x}}_i, \mathbf{y}_i \in \mathbb{R}^3$ is the RGB vector of the $i$-th pixel.

\subsubsection{GT-Mean Loss~\citep{liao2025gtmean}}

GT-Mean computes a \emph{single scalar} gain from the ratio of the global means of the ground truth and prediction:
\begin{equation}
    c_{\text{GM}} = \frac{\mu(\mathbf{I}_{\text{gt}})}{\mu(\hat{\mathbf{I}})},
    \quad\text{where}\quad
    \mu(\mathbf{A}) = \frac{1}{3N}\sum_{c=1}^{3}\sum_{i=1}^{N} A_{c,i},
    \label{eq:gtmean_scalar}
\end{equation}
i.e., $\mu(\cdot)$ averages over \emph{all} pixels \emph{and} all three color channels jointly, producing a single number. The aligned prediction and the GT-Mean loss are then:
\begin{equation}
    \hat{\mathbf{I}}_{\text{GM}} = c_{\text{GM}} \cdot \hat{\mathbf{I}},
    \qquad
    \mathcal{L}_{\text{GT-Mean}} = \bigl\| \hat{\mathbf{I}}_{\text{GM}} - \mathbf{I}_{\text{gt}} \bigr\|.
    \label{eq:gtmean_loss}
\end{equation}

\noindent In matrix form, this is equivalent to applying the alignment transform $\hat{\mathbf{x}}_i \mapsto c_{\text{GM}}\, \mathbf{E}\, \hat{\mathbf{x}}_i + \mathbf{0}$, where $\mathbf{E}$ is the $3{\times}3$ identity matrix:
\begin{equation}
    \underbrace{\mathbf{C}_{\text{GM}}}_{3\times 3} = c_{\text{GM}} \begin{pmatrix} 1 & 0 & 0 \\ 0 & 1 & 0 \\ 0 & 0 & 1 \end{pmatrix}, \qquad
    \underbrace{\mathbf{b}_{\text{GM}}}_{3\times 1} = \begin{pmatrix} 0 \\ 0 \\ 0 \end{pmatrix}.
    \label{eq:gtmean_matrix}
\end{equation}
Thus, GT-Mean has \textbf{1 free parameter} ($c_{\text{GM}}$): it applies the \emph{same} multiplicative factor to every pixel and every color channel, with \emph{no} additive bias.

\subsubsection{Photometric Alignment Loss (Ours)}

PAL models the photometric discrepancy as a \emph{full affine color transformation} with a $3{\times}3$ matrix $\mathbf{C}$ and a $3{\times}1$ bias vector $\mathbf{b}$:
\begin{equation}
    \mathbf{y}_i \approx \mathbf{C}\,\hat{\mathbf{x}}_i + \mathbf{b},
    \qquad i = 1, \dots, N.
    \label{eq:pal_model}
\end{equation}
The optimal parameters are obtained by ridge-regularized least squares (derivation in Section~\ref{sec:derivation} of this supplement):
\begin{equation}
    \mathbf{C}^{*} = \mathrm{Cov}(\mathbf{I}_{\text{gt}},\hat{\mathbf{I}})\bigl(\mathrm{Cov}(\hat{\mathbf{I}},\hat{\mathbf{I}})+\epsilon\,\mathbf{E}\bigr)^{-1},
    \qquad
    \mathbf{b}^{*} = \mu_{\text{gt}} - \mathbf{C}^{*}\mu_{\hat{\mathbf{I}}},
    \label{eq:pal_solution_supp}
\end{equation}
where $\mu_{\hat{\mathbf{I}}}, \mu_{\text{gt}} \in \mathbb{R}^3$ are the \emph{per-channel} means (unlike GT-Mean's scalar mean). The aligned prediction and the PAL loss are then:
\begin{equation}
    \hat{\mathbf{I}}_{\text{PAL}} = \mathbf{C}^{*}\hat{\mathbf{I}} + \mathbf{b}^{*},
    \qquad
    \mathcal{L}_{\text{PAL}} = \bigl\| \hat{\mathbf{I}}_{\text{PAL}} - \mathbf{I}_{\text{gt}} \bigr\|.
    \label{eq:pal_loss_supp}
\end{equation}
In full matrix form, $\mathbf{C}^{*}$ has \textbf{9 free parameters} (including off-diagonal entries that capture cross-channel coupling) and $\mathbf{b}^{*}$ has \textbf{3 free parameters} (additive per-channel biases), totaling \textbf{12 free parameters}:
\begin{equation}
    \underbrace{\mathbf{C}^{*}}_{3\times 3} = \begin{pmatrix} c_{rr} & c_{rg} & c_{rb} \\ c_{gr} & c_{gg} & c_{gb} \\ c_{br} & c_{bg} & c_{bb} \end{pmatrix}, \qquad
    \underbrace{\mathbf{b}^{*}}_{3\times 1} = \begin{pmatrix} b_r \\ b_g \\ b_b \end{pmatrix}.
    \label{eq:pal_matrix}
\end{equation}

\subsection{What GT-Mean Cannot Capture}

Since GT-Mean applies a single scalar to all channels identically, it cannot model any phenomenon where the three color channels behave differently:

\begin{itemize}[leftmargin=*]
    \item \textbf{Per-channel gain differences.} Exposure and sensor response often affect R, G, B channels non-uniformly. GT-Mean's scalar applies the same correction to all three channels ($c_{rr}{=}c_{gg}{=}c_{bb}{=}c_{\text{GM}}$), leaving per-channel gain discrepancies unresolved.
    \item \textbf{White-balance shifts.} These introduce off-diagonal terms in $\mathbf{C}$ (e.g., $c_{rb} \neq 0$ for a warm-to-cool shift). GT-Mean's $\mathbf{C}_{\text{GM}}$ has \emph{all zeros} off the diagonal, so it cannot model any cross-channel coupling.
    \item \textbf{Additive color biases.} Black-level offsets or ambient light require $\mathbf{b} \neq \mathbf{0}$. GT-Mean is purely multiplicative ($\mathbf{b}_{\text{GM}} = \mathbf{0}$) and cannot capture additive shifts.
    \item \textbf{Color-temperature variations.} These combine both per-channel multiplicative and cross-channel effects. As shown in Figure~2 of the main paper, GT-Mean leaves substantial color residuals in such cases, while PAL's full affine model closely recovers the reference color.
\end{itemize}

\subsection{Estimator Bias of GT-Mean}
From the statistical view, GT-Mean's ratio-of-means estimator $c_{\text{GM}} = \mu(\mathbf{I}_{\text{gt}}) / \mu(\hat{\mathbf{I}})$ is a biased estimator even for the best-fit scalar gain. The least-squares optimal scalar $c^* = \mathbb{E}[\hat{\mathbf{I}} \cdot \mathbf{I}_{\text{gt}}] / \mathbb{E}[\hat{\mathbf{I}}^2]$ differs from $c_{\text{GM}}$ unless the pixel intensities are uncorrelated, which is often violated for natural images.

\subsection{Task Generalizability}
GT-Mean was tailored for the low-light image enhancement domain, where the dominant photometric shift is an overall brightness difference. While effective in this context, it fails for tasks with more complex color discrepancy, for instance, underwater enhancement, as these tasks exhibit color-dependent and coupled photometric shifts that GT-Mean's scalar model cannot address.

\noindent\textbf{Empirical comparison on low-light enhancement.}
To provide a direct head-to-head comparison, we train all four LLIE backbones under three configurations: Baseline, +GT-Mean Loss~\citep{liao2025gtmean}, and +PAL (Ours), with all other settings identical. Results are shown in Table~\ref{tab:pal_vs_gtmean}.

\begin{table*}[t]
\centering
\setlength{\tabcolsep}{2pt}
\resizebox{\linewidth}{!}{
\begin{tabular}{l ccccc ccccc ccccc}
\toprule
\multirow{2}{*}{Methods} &
  \multicolumn{5}{c}{LOLv1} &
  \multicolumn{5}{c}{LOLv2-syn} &
  \multicolumn{5}{c}{LOLv2-real} \\
\cmidrule(lr){2-6} \cmidrule(lr){7-11} \cmidrule(lr){12-16}
 &
  PSNR & SSIM & LPIPS & IQA & IAA &
  PSNR & SSIM & LPIPS & IQA & IAA &
  PSNR & SSIM & LPIPS & IQA & IAA \\
\midrule
MIRNet~\citep{zamir2020learning}  & 20.57 & 0.769 & 0.254 & \underline{2.914}  & \textbf{1.855}  & 21.74 & 0.877 & 0.138 &  2.650  & 2.023 & 21.28 & 0.791 & \underline{0.355} & 2.256 & \underline{1.377} \\
+GT-Mean Loss  & \underline{20.79} & \underline{0.787} & \underline{0.233} & 2.892  & 1.728  & \underline{21.92} & \underline{0.888} & \underline{0.118} & \underline{2.840}    & \underline{2.084}  & \underline{21.77} & \underline{0.800} & 0.365 & \underline{2.285}  & 1.331 \\
+PAL (Ours)   & \textbf{21.01} & \textbf{0.791} & \textbf{0.228} &   \textbf{2.923}   & \underline{1.729}  & \textbf{22.13} & \textbf{0.890} & \textbf{0.117} & \textbf{2.870}   & \textbf{2.091}  & \textbf{22.32} & \textbf{0.816} & \textbf{0.348} & \textbf{2.535}   & \textbf{1.469}  \\
\midrule
Uformer~\citep{wang2022uformer}   & 18.85 & 0.751 & 0.288 & \underline{2.751} & \underline{1.661}  & 21.50 & 0.884 & 0.120 & 2.919 & 2.058  & 19.80 & 0.714 & 0.346 & 2.135  & 1.387 \\
+GT-Mean Loss  & \underline{19.01} & \underline{0.759} & \underline{0.279} & 2.728  & 1.558  & \underline{21.62} & \textbf{0.894} & \textbf{0.107} & \textbf{2.996}  & \textbf{2.104}  & \underline{19.92} & \underline{0.729} & \underline{0.331} & \underline{2.154}  & \underline{1.416}  \\
+PAL (Ours)    & \textbf{19.31} & \textbf{0.767} & \textbf{0.247} & \textbf{2.858}   & \textbf{1.666} & \textbf{21.79} & \underline{0.890} & \underline{0.112} & \underline{2.952}  & \underline{2.063}  & \textbf{20.12} & \textbf{0.738} & \textbf{0.327} & \textbf{2.780}   & \textbf{1.649}  \\
\midrule
Retinexformer~\citep{cai2023retinexformer}  & 23.40 & 0.822 & 0.269 & 3.148  & 1.980  & 25.48 & 0.930 & \underline{0.101} & 2.404  & 2.096  & 21.69 & 0.846 & 0.276 & 3.163   & 1.962  \\
+GT-Mean Loss  & \underline{24.03} & \underline{0.842} & \underline{0.240} & \underline{3.412}  & \underline{2.094}  & \underline{25.87} & \underline{0.940} & \textbf{0.083} & \textbf{2.438}  & \textbf{2.189}  & \underline{21.87} & \underline{0.849} & \textbf{0.253} & \underline{3.398}  & \underline{2.057}      \\
+PAL (Ours)    & \textbf{24.53} & \textbf{0.847} & \textbf{0.239} & \textbf{3.533} & \textbf{2.122}   & \textbf{26.01} & \textbf{0.941} & \textbf{0.083} & \underline{2.435}  & \underline{2.162}  & \textbf{22.73} & \textbf{0.864} & \underline{0.265} & \textbf{3.499}  & \textbf{2.077}  \\
\midrule
CID-Net~\citep{yan2025hvi}   & \underline{23.97} & \underline{0.849} & \underline{0.104} & 3.791  & 2.071  & 25.44 & 0.935 & \underline{0.047} & 3.299  & 2.171 & 23.19 & 0.857 & 0.136 & 3.699  & 2.042  \\
+GT-Mean Loss  & 23.72 & 0.838 & 0.105 & \textbf{3.993}  & \textbf{2.133}  & \underline{25.53} & \underline{0.936} & \textbf{0.045} & \textbf{3.410}  & \textbf{2.198}  & \underline{23.39} & \underline{0.863} & \underline{0.120} & \underline{3.917}  & \underline{2.089}  \\
+PAL (Ours)    & \textbf{24.13}  & \textbf{0.854} & \textbf{0.099} & \underline{3.923}  & \underline{2.104}  & \textbf{25.84} & \textbf{0.937} & \textbf{0.045} & \underline{3.373}  & \underline{2.185}  & \textbf{23.95} & \textbf{0.870} & \textbf{0.112} & \textbf{3.938}  & \textbf{2.103}  \\
\bottomrule
\end{tabular}
}
\caption{Direct comparison of Baseline, GT-Mean Loss~\citep{liao2025gtmean}, and PAL (Ours) on LOL datasets. Best and second-best results are in \textbf{bold} and \underline{underlined}. PAL achieves the best PSNR and SSIM across all backbones and datasets, demonstrating that the full affine alignment consistently outperforms mean-based alignment even on GT-Mean's home domain (LLIE).}
\label{tab:pal_vs_gtmean}
\end{table*}

Even on low-light enhancement, GT-Mean's home domain, PAL consistently outperforms GT-Mean across all four backbones on the primary fidelity metrics (PSNR, SSIM). The improvements are particularly notable on LOLv2-real, where the photometric inconsistency is strongest: PAL outperforms GT-Mean by +0.55~dB (MIRNet), +0.20~dB (Uformer), +0.86~dB (Retinexformer), and +0.56~dB (CID-Net) in PSNR, with corresponding SSIM gains. On CID-Net, GT-Mean actually \emph{degrades} PSNR on LOLv1 relative to the baseline (23.72 vs. 23.97~dB), likely because its additive correction interferes with CID-Net's learnable color space. PAL, by contrast, improves CID-Net across all three datasets. These results confirm that PAL's additional modeling capacity (multiplicative gains and cross-channel coupling) translates into measurable improvements even in the specific domain for which GT-Mean was designed.

\section{Implementation of PAL}

We provide a PyTorch implementation of PAL below. The core computation is minimal and introduces no learnable parameters.

\begin{lstlisting}[style=mypython]
def pal_loss(pred, gt, alpha=0.6, eps=1e-3):
    """Photometric Alignment Loss.
    Args: pred, gt: (B,3,H,W) in [0,1].
    Returns: scalar loss."""
    B, C, H, W = pred.shape
    # Flatten to (B, N, 3) pixels
    P = pred.permute(0,2,3,1).reshape(B,-1,3)
    T = gt.permute(0,2,3,1).reshape(B,-1,3)
    # Design matrix X = [P, 1] -> (B, N, 4)
    X = torch.cat([P, P.new_ones(B,P.shape[1],1)], -1)
    # Ridge regression: W = (XtX + eps*I)^{-1} XtT
    XtX = X.transpose(1,2) @ X          # (B,4,4)
    XtT = X.transpose(1,2) @ T          # (B,4,3)
    I4  = torch.eye(4, device=pred.device).unsqueeze(0)
    W   = torch.linalg.solve(XtX + eps * I4, XtT)
    M   = W.transpose(1,2)              # (B,3,4)
    # Apply alignment (stop-gradient on M)
    M = M.detach()
    Xf = torch.cat([pred.reshape(B,3,-1),
                    pred.new_ones(B,1,H*W)], 1)
    aligned = (M @ Xf).reshape(B, 3, H, W)
    # Combined loss
    return alpha * F.l1_loss(aligned, gt)
\end{lstlisting}

\section{Extended Theoretical Analysis}
\subsection{Derivation of the Closed-Form Alignment}
\label{sec:derivation}

\subsubsection{Problem Formulation}
Let $\mathbf{I} \in \mathbb{R}^{3 \times N}$ denote the predicted image and $\mathbf{I}_{\text{gt}} \in \mathbb{R}^{3 \times N}$ denote the ground truth reference image, where $N$ represents the total number of pixels (flattened spatial dimensions). Each column $\mathbf{x}_i$ of $\mathbf{I}$ and $\mathbf{y}_i$ of $\mathbf{I}_{\text{gt}}$ represents the RGB vector of the $i$-th pixel.

We model the photometric relationship as an affine transformation $\mathbf{y}_i \approx \mathbf{C}\mathbf{x}_i + \mathbf{b}$, where $\mathbf{C} \in \mathbb{R}^{3 \times 3}$ is the linear transformation matrix capturing exposure and color coupling, and $\mathbf{b} \in \mathbb{R}^{3 \times 1}$ is the bias capturing global offsets. To ensure numerical stability and prevent overfitting to monochromatic regions (where the color covariance would be singular), we employ Ridge Regression ($\ell_2$ regularization) on the transformation matrix $\mathbf{C}$. Our objective function is to minimize the regularized Mean Squared Error:

\begin{equation}
    \mathcal{J}(\mathbf{C}, \mathbf{b}) = \sum_{i=1}^{N} \| (\mathbf{C}\mathbf{x}_i + \mathbf{b}) - \mathbf{y}_i \|_2^2 + \lambda \|\mathbf{C}\|_F^2,
    \label{eq:objective}
\end{equation}
where $\|\cdot\|_F$ denotes the Frobenius norm and $\lambda$ is the regularization coefficient.

\subsection{Optimal Bias}
First, we solve for the optimal bias $\mathbf{b}^*$ by taking the partial derivative of Eq.~\ref{eq:objective} with respect to $\mathbf{b}$ and setting it to zero:

\begin{equation}
    \frac{\partial \mathcal{J}}{\partial \mathbf{b}} = \sum_{i=1}^{N} 2(\mathbf{C}\mathbf{x}_i + \mathbf{b} - \mathbf{y}_i) = 0.
\end{equation}

\noindent  Rearranging the terms, we obtain:
\begin{equation}
    \sum_{i=1}^{N} \mathbf{y}_i = \mathbf{C} \sum_{i=1}^{N} \mathbf{x}_i + \sum_{i=1}^{N} \mathbf{b}.
\end{equation}

\noindent  Dividing by $N$, we express the relationship in terms of the expected values (means) of the images, denoted as $\mu_{\mathbf{I}} = \frac{1}{N}\sum \mathbf{x}_i$ and $\mu_{\text{gt}} = \frac{1}{N}\sum \mathbf{y}_i$:
\begin{equation}
    \mu_{\text{gt}} = \mathbf{C} \mu_{\mathbf{I}} + \mathbf{b}.
\end{equation}

\noindent Thus, the optimal bias is determined by the alignment of the centroids:
\begin{equation}
    \mathbf{b}^* = \mu_{\text{gt}} - \mathbf{C}\mu_{\mathbf{I}}.
    \label{eq:optimal_b}
\end{equation}

\subsubsection{Optimal Transformation Matrix}
Substituting $\mathbf{b}^*$ back into the objective function eliminates $\mathbf{b}$ and centers the data. Let $\bar{\mathbf{x}}_i = \mathbf{x}_i - \mu_{\mathbf{I}}$ and $\bar{\mathbf{y}}_i = \mathbf{y}_i - \mu_{\text{gt}}$ be the mean-centered pixels. The objective function simplifies to:

\begin{equation}
    \mathcal{J}(\mathbf{C}) = \sum_{i=1}^{N} \| \mathbf{C}\bar{\mathbf{x}}_i - \bar{\mathbf{y}}_i \|_2^2 + \lambda \|\mathbf{C}\|_F^2.
\end{equation}

\noindent We can express this in matrix notation. Let $\bar{\mathbf{I}} \in \mathbb{R}^{3 \times N}$ and $\bar{\mathbf{I}}_{\text{gt}} \in \mathbb{R}^{3 \times N}$ be the matrices of centered pixels. The objective becomes:
\begin{equation}
    \mathcal{J}(\mathbf{C}) = \| \mathbf{C}\bar{\mathbf{I}} - \bar{\mathbf{I}}_{\text{gt}} \|_F^2 + \lambda \|\mathbf{C}\|_F^2.
\end{equation}

\noindent Using the trace properties of the Frobenius norm ($\| \mathbf{A} \|_F^2 = \text{Tr}(\mathbf{A}^\top \mathbf{A})$), we expand the term:
\begin{equation}
\begin{aligned}
    \mathcal{J}(\mathbf{C}) &= \text{Tr}\left( (\mathbf{C}\bar{\mathbf{I}} - \bar{\mathbf{I}}_{\text{gt}})^\top (\mathbf{C}\bar{\mathbf{I}} - \bar{\mathbf{I}}_{\text{gt}}) \right) + \lambda \text{Tr}(\mathbf{C}^\top \mathbf{C}) \\
    &= \text{Tr}\left( \bar{\mathbf{I}}^\top \mathbf{C}^\top \mathbf{C} \bar{\mathbf{I}} - \bar{\mathbf{I}}^\top \mathbf{C}^\top \bar{\mathbf{I}}_{\text{gt}} - \bar{\mathbf{I}}_{\text{gt}}^\top \mathbf{C} \bar{\mathbf{I}} + \bar{\mathbf{I}}_{\text{gt}}^\top \bar{\mathbf{I}}_{\text{gt}} \right) \\
    & + \lambda \text{Tr}(\mathbf{C}^\top \mathbf{C}).
\end{aligned}
\end{equation}

\noindent Taking the derivative with respect to $\mathbf{C}$ and setting it to zero:
\begin{equation}
    \frac{\partial \mathcal{J}}{\partial \mathbf{C}} = 2\mathbf{C}\bar{\mathbf{I}}\bar{\mathbf{I}}^\top - 2\bar{\mathbf{I}}_{\text{gt}}\bar{\mathbf{I}}^\top + 2\lambda \mathbf{C} = 0.
\end{equation}

\noindent Rearranging to solve for $\mathbf{C}$:
\begin{equation}
    \mathbf{C}(\bar{\mathbf{I}}\bar{\mathbf{I}}^\top + \lambda \mathbf{E}) = \bar{\mathbf{I}}_{\text{gt}}\bar{\mathbf{I}}^\top,
\end{equation}
where $\mathbf{E}$ is the $3 \times 3$ identity matrix.

  We recognize that the term $\frac{1}{N}\bar{\mathbf{I}}\bar{\mathbf{I}}^\top$ is the covariance matrix of the predicted image, $\text{Cov}(\mathbf{I}, \mathbf{I})$, and $\frac{1}{N}\bar{\mathbf{I}}_{\text{gt}}\bar{\mathbf{I}}^\top$ is the cross-covariance matrix, $\text{Cov}(\mathbf{I}_{\text{gt}}, \mathbf{I})$. Dividing the equation by $N$ and letting $\epsilon = \lambda/N$, we arrive at the final closed-form solution:

\begin{equation}
    \mathbf{C}^* = \text{Cov}(\mathbf{I}_{\text{gt}}, \mathbf{I}) \left( \text{Cov}(\mathbf{I}, \mathbf{I}) + \epsilon \mathbf{E} \right)^{-1}.
    \label{eq:c_star_ridge_supp}
\end{equation}

\noindent This matches Eq.~(6) in the main paper. The term $\epsilon \mathbf{E}$ ensures that the matrix inverse exists and is numerically stable even when the input image $\mathbf{I}$ has low color variance (rank-deficient covariance).

\subsubsection{Cross-Term Under Ridge Regularization}

Proposition~1 in the main paper establishes exact orthogonal decomposition of the MSE loss under unregularized least-squares alignment. In practice, we use ridge regularization ($\epsilon > 0$) for numerical stability. This modifies the first-order optimality conditions for $\mathbf{C}^*$ (while the bias optimality is unchanged since $\mathbf{b}$ is not regularized):
\begin{equation}
    \textstyle\sum_{i}\boldsymbol{\Delta}_{s}^{(i)} = \mathbf{0}, \qquad
    \textstyle\sum_{i}\boldsymbol{\Delta}_{s}^{(i)}\,\hat{\mathbf{I}}^{(i)\!\top} = \lambda\,\mathbf{C}^{*},
    \label{eq:ridge_opt_cond}
\end{equation}
where $\lambda = N\epsilon$ is the unnormalized regularization term. Compared to the unregularized case (Eq.~(3) of the main paper, where the right-hand side is $\mathbf{0}$), the structural residual is no longer exactly orthogonal to the prediction. Expanding the cross-term:
\begin{align}
    \textstyle\sum_{i}\langle\boldsymbol{\Delta}_{p}^{(i)},\boldsymbol{\Delta}_{s}^{(i)}\rangle
    &= \mathrm{tr}\!\bigl[(\mathbf{C}^{*}\!-\!\mathbf{E})^{\!\top}\!\textstyle\sum_{i}\boldsymbol{\Delta}_{s}^{(i)}\hat{\mathbf{I}}^{(i)\!\top}\bigr]
       + \mathbf{b}^{*\!\top}\!\textstyle\sum_{i}\boldsymbol{\Delta}_{s}^{(i)} \nonumber\\
    &= \mathrm{tr}\!\bigl[(\mathbf{C}^{*}\!-\!\mathbf{E})^{\!\top}\!\cdot\lambda\,\mathbf{C}^{*}\bigr] + 0 \nonumber\\
    &= \lambda\bigl(\|\mathbf{C}^{*}\|_{F}^{2} - \mathrm{tr}(\mathbf{C}^{*})\bigr).
    \label{eq:ridge_cross}
\end{align}
The MSE therefore decomposes as:
\begin{equation}
    \textstyle\sum_{i}\bigl\|\mathbf{I}_{\text{gt}}^{(i)}-\hat{\mathbf{I}}^{(i)}\bigr\|^{2}
    = \textstyle\sum_{i}\bigl\|\boldsymbol{\Delta}_{p}^{(i)}\bigr\|^{2}
    + \textstyle\sum_{i}\bigl\|\boldsymbol{\Delta}_{s}^{(i)}\bigr\|^{2}
    + \lambda\bigl(\|\mathbf{C}^{*}\|_{F}^{2} - \mathrm{tr}(\mathbf{C}^{*})\bigr).
    \label{eq:ridge_decomp}
\end{equation}

 The cross-term is proportional to $\lambda$ and vanishes continuously. In our implementation, the regularization is applied to the unnormalized Gram matrix $\bar{\mathbf{I}}\bar{\mathbf{I}}^{\top}$ with $\lambda = 0.001$ (equivalently, $\epsilon = \lambda/N$ in the covariance form of Eq.~\ref{eq:c_star_ridge_supp}), so the cross-term is $O(\lambda) = O(10^{-3})$, negligible relative to the photometric and structural energies that scale as $O(N)$.
As a result, the approximate decomposition holds, and the gradient dominance analysis from the main paper remains valid. The regularization bias also has a well-understood effect on the alignment itself. Ridge shrinks each eigendirection of $\mathbf{C}^{*}$ by a factor of $\lambda_k / (\lambda_k + \epsilon)$, where $\lambda_k$ are the eigenvalues of $\mathrm{Cov}(\hat{\mathbf{I}}, \hat{\mathbf{I}})$. For well-conditioned images where $\lambda_k \gg \epsilon$, this shrinkage is negligible; for ill-conditioned cases (near-monochromatic regions), the shrinkage prevents the degenerate solutions that would arise from inverting a singular covariance.

\subsubsection{Gradient Dominance Under $\ell_1$ Loss}

The orthogonal decomposition in Proposition~1 of the main paper relies on the quadratic structure of the $\ell_2$ norm. The $\ell_1$ loss $\mathcal{L}_{L1} = \sum_{i}\|\mathbf{I}_{\text{gt}}^{(i)} - \hat{\mathbf{I}}^{(i)}\|_{1}$ does not admit an analogous exact decomposition, since in general $\|a+b\|_1 \neq \|a\|_1 + \|b\|_1$. However, we show that the gradient dominance pathology persists, and in fact is \emph{more severe}, under $\ell_1$.

\noindent\textbf{Gradient of $\ell_1$.}
The per-pixel, per-channel gradient of the $\ell_1$ loss is:
\begin{equation}
    \frac{\partial}{\partial \hat{I}_{c}^{(i)}} \bigl|I_{\text{gt},c}^{(i)} - \hat{I}_{c}^{(i)}\bigr| = -\mathrm{sign}\!\bigl(\Delta_{p,c}^{(i)} + \Delta_{s,c}^{(i)}\bigr),
    \label{eq:L1_grad}
\end{equation}
where $c$ indexes the color channel. Unlike $\ell_2$, where the gradient magnitude is proportional to the residual, the $\ell_1$ gradient has \emph{unit magnitude} at every pixel regardless of the error size. The only information carried by the gradient is its \emph{sign}, which is determined by whichever component, photometric or structural, has a larger absolute value at that pixel-channel.

\noindent\textbf{Photometric dominance of gradient direction.}
As established in the main paper, $\boldsymbol{\Delta}_{p}^{(i)}$ is spatially \emph{dense} (non-zero at every pixel, since it is a global affine function of the prediction), while $\boldsymbol{\Delta}_{s}^{(i)}$ is spatially \emph{sparse} (concentrated on edges, textures, and fine structures). For the majority of pixels in smooth regions:
\begin{equation}
    |\Delta_{p,c}^{(i)}| \gg |\Delta_{s,c}^{(i)}| \;\implies\; 
    \mathrm{sign}\!\bigl(\Delta_{p,c}^{(i)} + \Delta_{s,c}^{(i)}\bigr) = \mathrm{sign}\!\bigl(\Delta_{p,c}^{(i)}\bigr).
    \label{eq:sign_dom}
\end{equation}
Hence, the $\ell_1$ gradient direction is determined by the photometric component for most pixels, and only the sparse minority where structural error exceeds photometric error contributes a content-relevant gradient.

\noindent\textbf{$\ell_1$ amplifies the pathology.}
Under $\ell_2$, a photometric mismatch of magnitude $\delta$ at $N$ pixels produces a total gradient energy of $O(N\delta^2)$; a structural mismatch of magnitude $\Delta$ at $M$ pixels produces $O(M\Delta^2)$. The gradient energy ratio is $N\delta^2 / (M\Delta^2)$. Under $\ell_1$, however, both components produce unit-magnitude gradients, so the ratio is simply $N_p / N_s$, where $N_p$ is the number of pixels where photometric error dominates the sign, and $N_s = N - N_p$ is the complement. Since $N_p \gg N_s$ (photometric error is dense), \emph{the $\ell_1$ gradient direction is dominated by the photometric component even when its magnitude is smaller}. This analysis demonstrates that PAL is well-motivated under $\ell_1$ training as under $\ell_2$. In our experiments, we apply PAL when the baseline model is trained with $\ell_1$ losses. This empirically confirms the phenomenon.

\section{Additional Per-Pair Photometric Analysis}

To complement the per-pair scatter plots in Figure~2 of the main paper (LOLv2-Real and RESIDE-SOTS), we provide scatter plots for nine datasets used for training in our experiments in Figure~\ref{fig:dataset_analysis}. Each subplot shows the per-channel (R, G, B) mean intensity of the input versus the ground truth for every image pair in the training set, along with per-channel linear fits. If the photometric mapping were consistent and identity-preserving, all points would lie on the gray diagonal ($y = x$). Deviations from this diagonal, \emph{scatter} around the fitted lines, and \emph{separation} between the per-channel fits jointly quantify the severity of photometric inconsistency. For each pair (to reduce the number of points, for datasets with more than 1000 images, we plot the first 1000 images ) in the training set, we compute the spatial mean of each color channel independently, yielding a 3-dimensional summary $(\bar{r}, \bar{g}, \bar{b})$ for both the input and the ground truth.  Each channel is then plotted as a separate point (red, green, blue) in the input-mean vs.\ ground-truth-mean plane, so a dataset of $n$ pairs produces $3n$ points.  To visualize the density of overlapping points, we overlay Gaussian kernel density estimation (KDE) contours for each channel.

\noindent\textbf{Low-light enhancement (LOL-V2 Real, LOL-V2 Syn, LOL-V1).}
All three LLIE datasets exhibit the most pronounced inconsistency. The input means cluster near zero (underexposed), while the ground-truth means span a wide range, producing large deviations from the $y{=}x$ diagonal. This is the canonical example of \emph{task-intrinsic} photometric inconsistency identified in Section~3.1 of the main paper: different pairs demand different brightness and color-temperature mappings depending on capture conditions and photographer intent. The wide per-pair scatter around each regression line means that the pixel-wise loss receives conflicting photometric supervision across pairs. Moreover, the per-channel regression lines have visibly different slopes, confirming that the inconsistency is not a uniform brightness shift but a channel-dependent color transformation.

\noindent\textbf{Shadow removal (ISTD), Underwater enhancement (EUVP), and Image dehazing (RESIDE-SOTS).}
These three datasets illustrate how photometric inconsistency manifests across different task families with varying characteristics.
For ISTD, points cluster near the diagonal with ground-truth means slightly above input means, consistent with shadow-free images being brighter. The per-channel regression lines diverge in slope, say, the B channel line is notably steeper than R and G, confirming that shadow attenuation is wavelength-dependent~\citep{shadowhack} and introduces channel-coupled biases beyond a uniform darkening. This per-pair, per-channel variation creates the same conflicting supervision identified in the main paper.
For EUVP, R-channel points are systematically shifted below G and B, reflecting the selective attenuation of red wavelengths in underwater imaging~\citep{8253820}. The per-channel regression lines are clearly separated with different slopes, and the substantial scatter around each line confirms per-pair variability. This channel-coupled behavior exemplifies \emph{task-intrinsic} inconsistency where physical degradation is inherently wavelength-dependent.
For RESIDE-SOTS, points lie below the diagonal (hazy inputs appear brighter than clean ground truth due to additive atmospheric scattering), and the regression lines exhibit non-zero intercepts. These \emph{acquisition-induced} offsets inject conflicting supervision into the pixel-wise loss. 

\noindent\textbf{All-weather restoration (Snow, Rain+Haze, Raindrop).}
The three all-weather subsets display qualitatively different photometric profiles: Snow training data shows tight clustering near the diagonal with moderate scatter since they are synthetic; Rain+Haze exhibits broader scatter with clear channel separation; Raindrop clusters tightly near the diagonal since raindrops cause localized rather than global photometric changes. When combined for all-in-one training, the network receives three distinct photometric profiles simultaneously, amplifying the conflicting supervision. As noted in the main paper (Section~2.1), multi-dataset training compounds the inconsistency because each constituent dataset has its own photometric profile, making the per-pair variation even wider.

\begin{figure}[t]
    \centering
    \includegraphics[width=\linewidth]{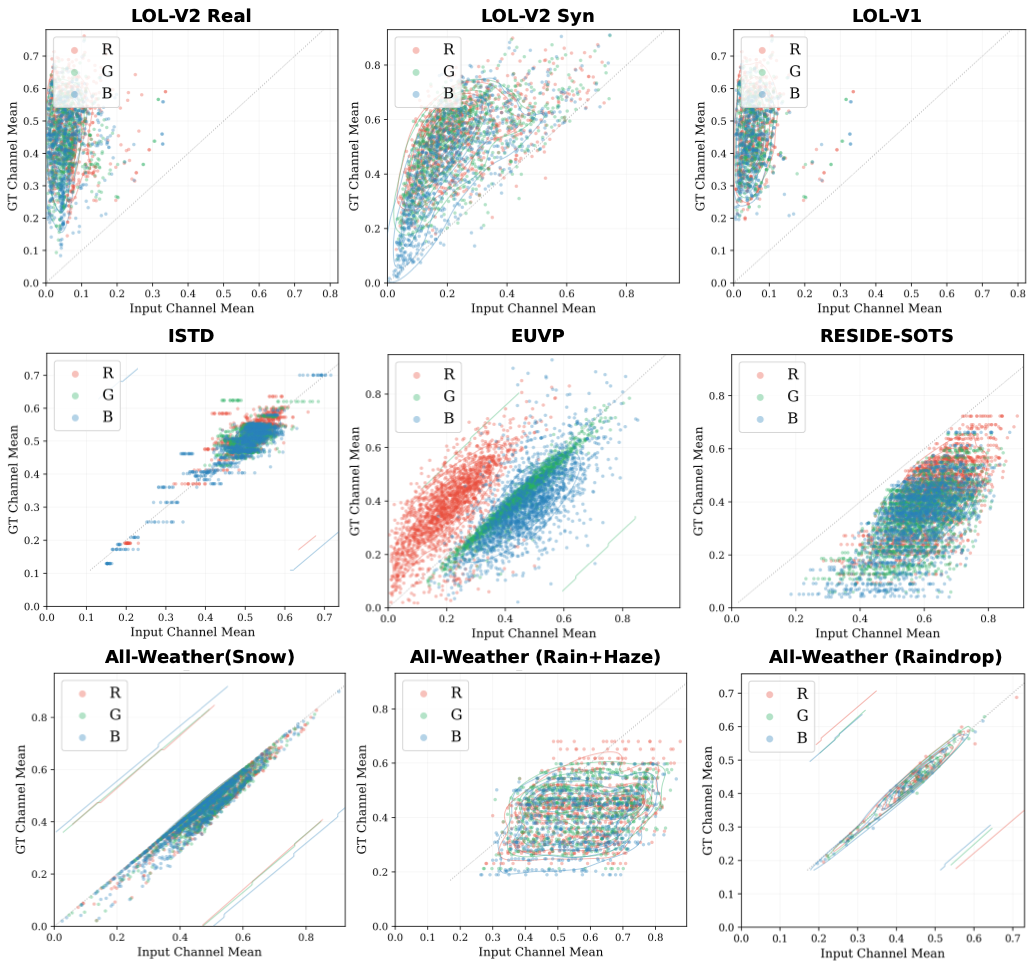}
    \caption{Per-pair photometric analysis across nine datasets spanning four task families. Each plot shows the per-channel (R/G/B) mean intensity of the input vs.\ ground truth for every training pair, with per-channel linear fits and KDE density contours. The gray dashed line denotes $y{=}x$ (identity). All datasets exhibit per-pair scatter away from any single trajectory, confirming the ubiquity of per-pair photometric inconsistency that injects conflicting supervision into pixel-wise losses (cf.\ Section~3.1 of the main paper). The separation between per-channel regression lines further demonstrates that the inconsistency is channel-dependent, requiring a full affine color model to discount.}
    \label{fig:dataset_analysis}
\end{figure}

\end{document}